\documentclass[10pt,twocolumn,letterpaper]{article}

\usepackage{titling}

\usepackage{wacv}              

\usepackage{graphicx}
\usepackage{amsmath}
\usepackage{amssymb}
\usepackage{booktabs}

\usepackage{multicol,multirow}
\usepackage{svg}
\usepackage{pifont}
\usepackage[para]{threeparttable}
\usepackage{nicematrix}
\usepackage{enumitem}
\usepackage{diagbox}
\usepackage{dblfloatfix}
\usepackage{algorithm}
\usepackage{algorithmicx}
\usepackage{chapterbib}
\usepackage{algpseudocode}
\usepackage{wrapfig}
\newcommand{\cmark}{\ding{51}}%
\newcommand{\edit}[1]{\textcolor{black}{#1}}
\usepackage[title]{appendix}

\usepackage[pagebackref,breaklinks,colorlinks]{hyperref}

\usepackage[capitalize]{cleveref}
\crefname{section}{Sec.}{Secs.}
\Crefname{section}{Section}{Sections}
\Crefname{table}{Table}{Tables}
\crefname{table}{Tab.}{Tabs.}

\def\wacvPaperID{628\vspace{-0.1cm}} 
\def\confName{WACV}
\def\confYear{2025}

\begin{document}

\title{Bidirectional Multi-Step Domain Generalization\\ for Visible-Infrared Person Re-Identification} 
\author{\textsuperscript{1}Mahdi Alehdaghi, \textsuperscript{2}Pourya Shamsolmoali, \textsuperscript{1}Rafael M. O. Cruz, and \textsuperscript{1}Eric Granger\\
\textsuperscript{1}LIVIA, ILLS, Dept. of Systems Engineering, ETS Montreal, Canada\\
\textsuperscript{2}Dept. of Computer Science, University of York, UK \\
{\tt\small mahdi.alehdaghi.1@ens.etsmtl.ca,} 
{\tt\small pshams55@gmail.com,} \\
{\tt\small \{rafael.menelau-cruz, eric.granger\}@etsmtl.ca}
}

\maketitle
\vspace{-1cm}

\begin{abstract}
\vspace{-0.3cm}
A key challenge in visible-infrared person re-identification (V-I ReID) is training a backbone model capable of effectively addressing the significant discrepancies across modalities. State-of-the-art methods that generate a single intermediate bridging domain are often less effective, as this generated domain may not adequately capture sufficient common discriminant information.
This paper introduces Bidirectional Multi-step Domain Generalization (BMDG), a novel approach for unifying feature representations across diverse modalities. BMDG creates multiple virtual intermediate domains by learning and aligning body part features extracted from both I and V modalities. In particular, our method aims to minimize the cross-modal gap in two steps.  
First, BMDG aligns modalities in the feature space by learning shared and modality-invariant body part prototypes from V and I images. Then, it generalizes the feature representation by applying bidirectional multi-step learning, which progressively refines feature representations in each step and incorporates more prototypes from both modalities. 
Based on these prototypes, multiple bridging steps enhance the feature representation.
Experiments\footnote{Our code is available at: \href{https:/alehdaghi.github.io/BMDG/}{alehdaghi.github.io/BMDG}} conducted on V-I ReID datasets indicate that our BMDG approach can outperform state-of-the-art part-based and intermediate generation methods, and can be integrated into other part-based methods to enhance their V-I ReID performance. 
\end{abstract}

 \vspace{-0.5cm}
\section{Introduction} \label{sec:intro}
\vspace{-0.25cm}



V-I ReID is a variant of person ReID that involves matching individuals across RGB and IR cameras. V-I ReID is challenging since it requires matching individuals with significant differences in appearance between V and I modality images. In this context, V-I ReID systems must train a discriminant feature extraction backbone to encode consistent and identifiable attributes in RGB and IR images.
Most state-of-the-art methods seek to learn global representations from the whole image by alignment at the image-level \cite{Wang_2019_ICCV_AlignGAN, kniaz2018thermalgan} or feature-level \cite{all-survey, Bi-Di_Center-Constrained, hetero-center, park2021learning} (see Fig. \ref{fig:short}(a)). Other methods extract global modality-invariant features by disentangling them from modality-specific information \cite{HI-CMD, paired-images2}. \edit{However, global feature-based approaches cannot compare local discriminative attributes, resulting in the loss of important discriminant cues about the person.}
To address this issue and focus on the unique information in different body regions, part-based approaches (see Fig. \ref{fig:short}(b)) extract local fine-grained feature representations through horizontal stripes, clustering, or attention mechanisms of spatially extracted features \cite{DDAG,cmSSFT,part1,wu2021Nuances}. 
However, given the domain discrepancies between the V and I modalities and the lack of matching cross-modal parts during training, these methods often learn modality-specific attributes for each part. This typically results in feature representations that are less modality-invariant, and thus provide limited effectiveness for cross-modal matching tasks.

\begin{figure*} [!t]
  \centering
  \begin{subfigure}{0.19\linewidth}
    \includegraphics[width=\linewidth]{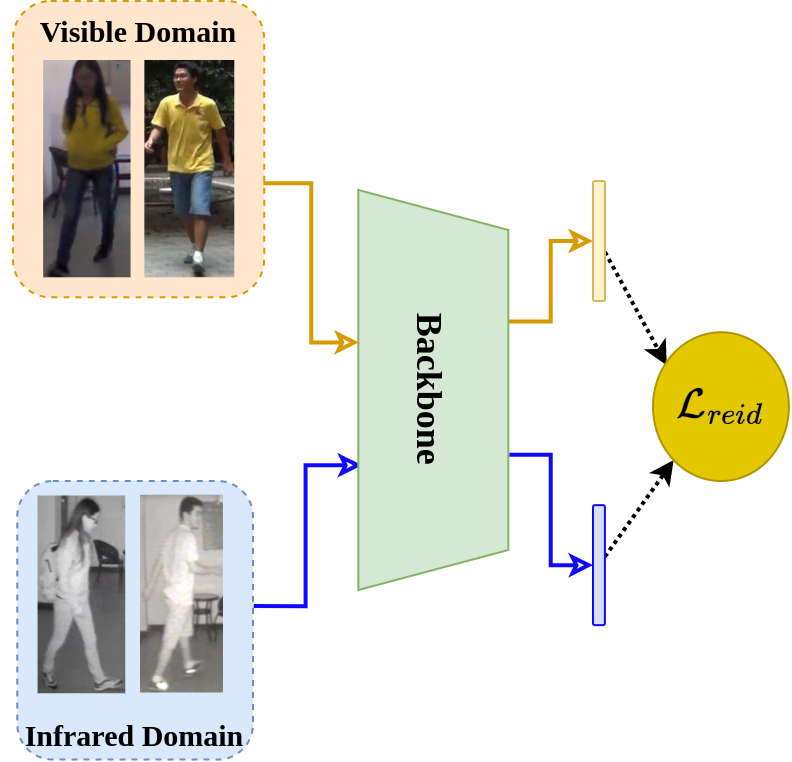}
    \caption{Global representation.}
    \label{fig:short-a}
  \end{subfigure}
  \hfill
  \begin{subfigure}{0.19\linewidth}
    \includegraphics[width=\linewidth]{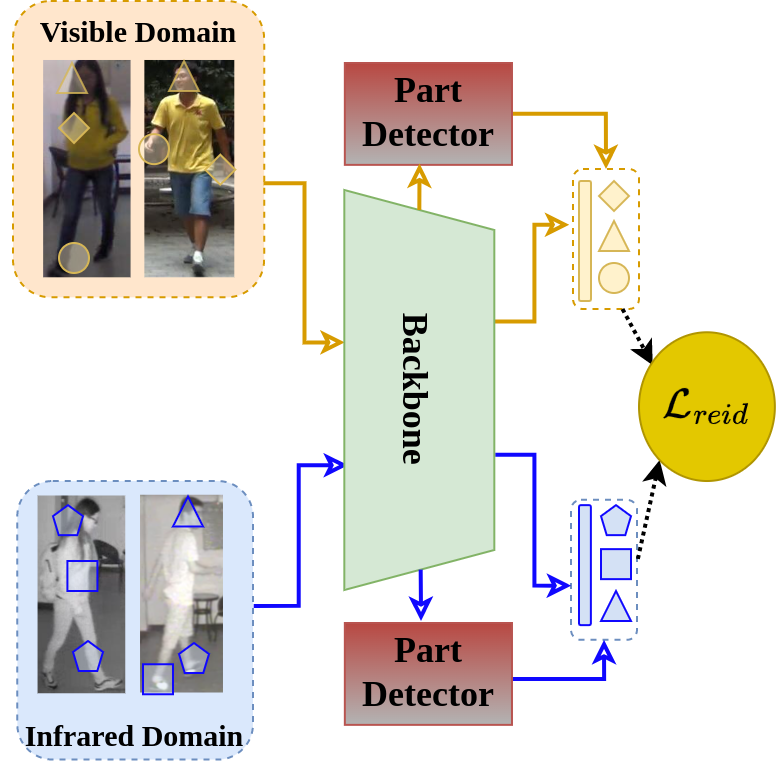}
    \caption{Part-based representation.}
    \label{fig:short-b}
  \end{subfigure}
  \hfill
  \begin{subfigure}{0.19\linewidth}
    \includegraphics[width=\linewidth]{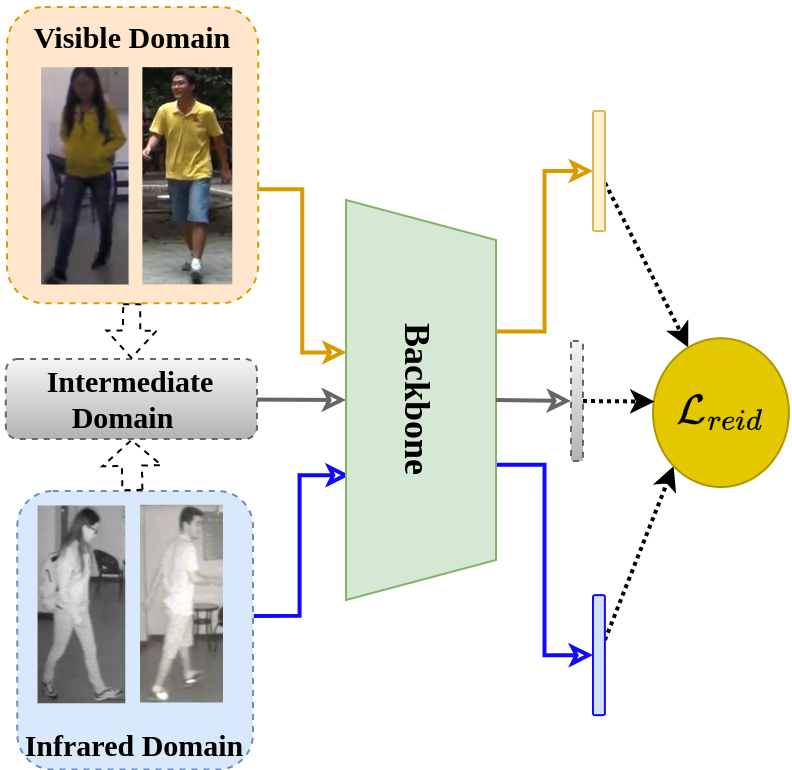}
    \caption{Intermediate modality.}
    \label{fig:short-c}
  \end{subfigure}
  \begin{subfigure}{0.35\linewidth}
    \includegraphics[width=\linewidth]{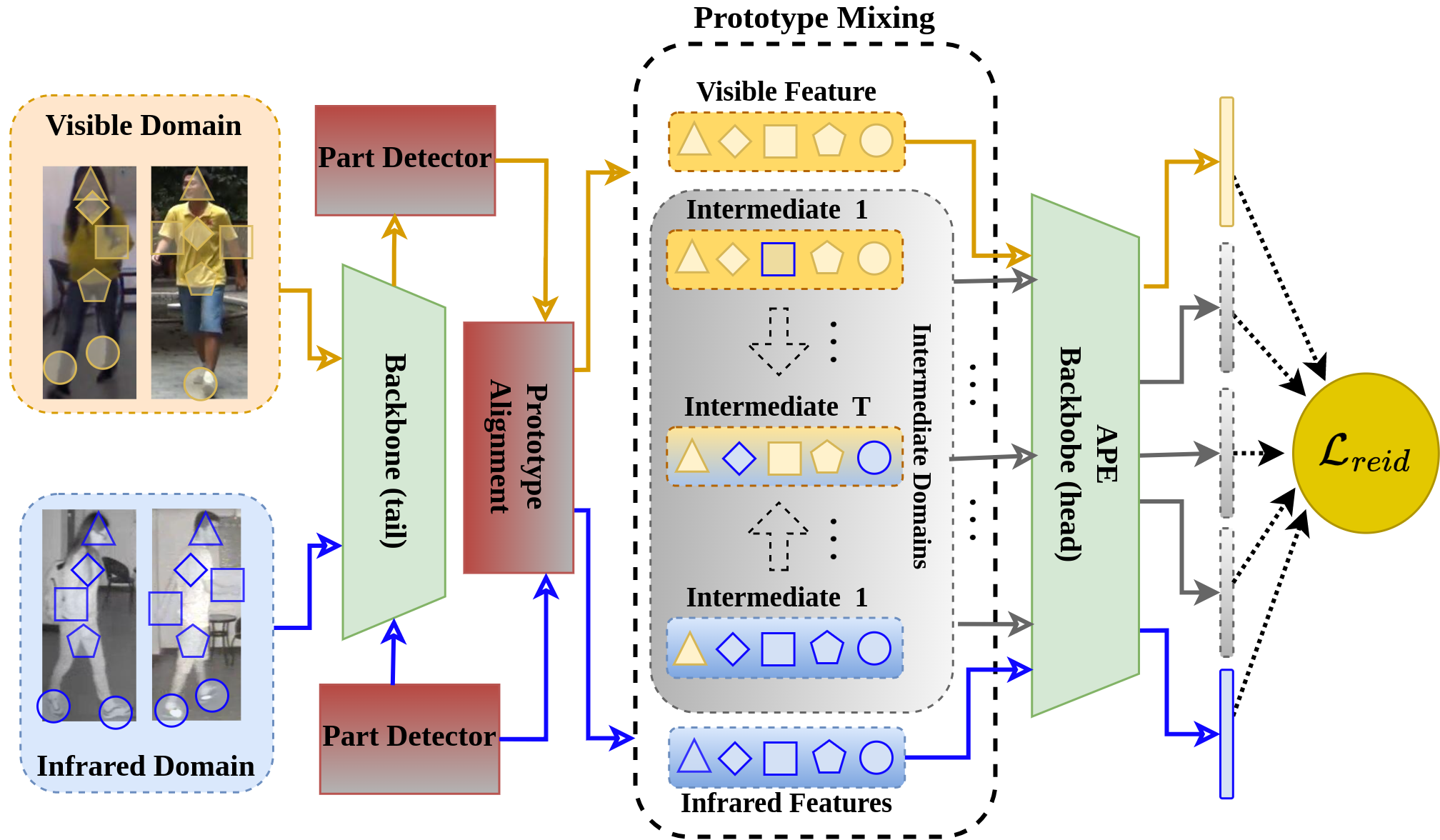}
    \caption{Our bidirectional multi-step domain generalization.}
    \label{fig:short-d}
  \end{subfigure}
  \vspace{-.15cm}
  \caption{A comparison of training architectures for V-I ReID. Approaches based on (a) a global features representation, and (b) a local part-based representation to preserve locality and global features. (c) Generation of an intermediate bridging domain to guide training. (d) Our BMGD extracts and combines prototypes from modalities in each step to gradually create multiple intermediate bridging domains.}
  \label{fig:short}
  \vspace{-.25cm}
\end{figure*}

Some methods leverage an intermediate modality to guide the training process such that the cross-modal domain gap is reduced \cite{xModal,HAT,randLUPI,wu2017rgb,zhang2021towards} (see Fig. \ref{fig:short}(c)). Intermediate modalities may be static images (e.g., gray-scale) \cite{HAT,randLUPI} or created images from a generative module learned by I and V modalities \cite{xModal,wu2017rgb,zhang2021towards}. The resulting intermediate representations can leverage shared relations between I and V images and improve robustness to modality changes. However, they generate images that tend to contain shared non-discriminative information, resulting in the loss of ID-aware information in the intermediate domain. Moreover, one-step training strategies only create one intermediate space between V and I. This cannot provide sufficient bridging information when modalities have significant discrepancies. Single-step approaches cannot effectively deal with these discrepancies during training. Therefore, we advocate for linking a series of intermediate domains with progressively smaller gaps.  

\edit{
To effectively bridge the large cross-modal domain gap while preserving local ID-discriminative information, we introduce Bidirectional Multi-step Domain Generalization (BMDG), a gradual, multi-step, part-based training strategy for person V-I ReID (see Fig. \ref{fig:short}(d)).  
At first, this strategy identifies similar ID-aware part prototypes\footnote{We define a {\it part prototype} as a feature representation corresponding to a specific body part in a cropped image of a person.} for each modality, and then it progressively creates intermediate steps by controlling the number of the part prototypes to be mixed. Generated images can become degraded, as discriminative details are lost over consecutive generation steps, making these images less useful for training. Therefore, our proposed BMDG creates multiple intermediate domains in the feature space by aligning, and then combining ID-informative subsets of features from both modalities.
}

Our BMGD training strategy is comprised of two parts. The prototype alignment module learns to identify and align part-prototype representations that contain discriminative information linked to body parts in I and V images. Then, the bidirectional multi-step learning progressively extracts auxiliary intermediate feature representations by mixing prototypes from both modalities at each step. For ID-aware and robust alignment between modalities, the extracted part-prototypes must meet three conditions: (1) they should be \textit{complementary} to other parts, (2) they should be \textit{interchangeable} and represent consistent parts, and (3) they should be \textit{discriminant} for person matching. Hierarchical contrastive learning is proposed to train the prototype alignment module such that these properties are respected without affecting ReID accuracy. Our novel idea is to progressively create intermediate bridging domains by mixing a growing number of body part prototypes extracted from I and V modalities. This allows leveraging common semantic information between modalities, along with discriminative information among individuals. 


Bidirectional multi-step learning relies on auxiliary intermediate feature spaces that are created at each step, by progressively mixing semantically consistent prototypes from each modality. While prototype-based approaches \cite{DDAG,liu2021sfanet,SAAI,ISP} concentrate solely on enhancing discriminatory representation using local descriptors, BMDG aligns and mixes local part information in a gradual training process to mitigate domain discrepancies within the representation space. Unlike the Part-Mix method \cite{PartMix23}, which utilizes a fixed proportion of mixed parts for auxiliary features, our approach employs bidirectional multi-step generation for multiple intermediate steps. This bolsters the robustness of V-I ReID training for samples with large discrepancies, especially at the beginning of training when the parts are inconsistent. Our strategy enables effective bridging of cross-modality gaps and allows gradual training by adjusting the proportion of information to be exchanged. We also propose a hierarchical prototype learning strategy, which is crucial to defining meaningful bridging steps. 

\noindent \textbf{Our main contributions are summarized as follows}.\\ 
\textbf{(1)} A BMDG method is proposed to train a backbone V-I ReID model that learns a common feature embedding by gradually reducing the gap between I and V modalities through multiple auxiliary intermediate steps.\\  
\textbf{(2)} A novel hierarchical prototype learning module is introduced to align discriminative and interchangeable part prototypes between modalities using two-level contrastive learning. These prototypes maintain semantic consistency to facilitate the exchange of V and I information, thus enabling the generation of intermediate bridging domains.\\  
\textbf{(3)} A bidirectional multi-step learning model is proposed to progressively combine prototypes from both modalities and create intermediate feature spaces at each step. It integrates attentive prototype embedding (APE), which refines these prototypes to extract attributes that are shared between modalities, thereby improving ReID accuracy. 
Indeed, BMDG can find better modality-shared attributes by optimizing the model on two levels: prototype alignment and bidirectional multi-step embedding modules.\\
\textbf{(4)} Extensive experiments on the challenging SYSU-MM01 \cite{SYSU}, RegDB \cite{regDB} and LLCM\cite{LLCM} datasets indicate that BMDG outperforms state-of-the-art methods for V-I person ReID. They also show that our framework can be integrated into any part-based V-I ReID method and can improve the performance of cross-modal retrieval applications.

 \vspace{-.15cm}
\section{Related Work}
\vspace{-.15cm}
In V-I ReID, deep backbone models are trained with labeled V and I images captured using RGB and IR cameras. During inference, cropped images of people are matched across camera modalities. That is, query I images are matched against V gallery images, or vice versa \cite{chen2021neural, cm-gan,fu2021cm, adv-modal, park2021learning, xu2021cross, HCML, hetero-center}. 
State-of-the-art methods can be classified into approaches for global or part-based representation, or that leverage an intermediate modality.

\noindent \textbf{(a) Global and Part-Based Representation Methods:} 
Representation approaches focus on training a backbone model to learn a discriminant representation that captures the essential features of each modality while exploiting their shared information \cite{HCML,cm-gan,Bi-Di_Center-Constrained,DZP,EDFL:journals/corr/abs-1907-09659,all-survey}. In \cite{all-survey,HAT,CAJ,park2021learning,JFLN}, authors only use global features through multimodal fusion. The absence of local information limits their accuracy in more challenging cases. To improve robustness, recent methods combine global information with local part-base features \cite{DDAG,cmSSFT,wu2021Nuances,SAAI,PartMix23}. For example, \cite{DDAG,cmSSFT} divides the spatial feature map into fixed horizontal sections and applies a weighted-part aggregation.  
\cite{SAAI,wu2021Nuances, CATA, ISP} dynamically detects regions in the spatial feature map to alleviate the misalignment of fixed horizontally divided body parts. 

\edit{
However, these part detectors might over-focus on specific parts, varying based on the modality or person. \cite{PartMix23} addresses this issue using Part-Mix data augmentation to regularize model part discovery. Additionally, parts may not correspond to semantically similar regions across different images. To tackle this, a bipartite graph-based correlation model between part regions has been proposed to maximize similarities \cite{CATA}. While part-based methods enhance representation ability, they often overlook shared information between similar semantic parts across different individuals or modalities, leading to inconsistencies in extracted part features. This hinders effective comparison or exchange. 
}
\edit{
Similar to the human perception system, distinguishing individuals by comparing attributes of similar body parts (e.g., head, torso, leg) learned from a diverse set of individuals, our approach seeks to capture and leverage this shared information for accurate matching. To achieve this, BMDG discovers different body parts as prototypes, disregarding the person's identity using hierarchical contrastive learning, and then leverages their ID-discriminative attributes for each part through part-based cross-entropy loss.}

\noindent \textbf{(b) Methods Based on an Intermediate Modality:} 
In recent years, V-I ReID methods have relied on an intermediate modality to address the significant gap between V and I modalities \cite{Fan2020CrossSpectrumDP, HAT, randLUPI, xModal,wu2017rgb,zhang2021towards,shape-Erase23,alehdaghi2023adaptive}. Using fixed intermediate modality like grayscale \cite{Fan2020CrossSpectrumDP,HAT} or random channels \cite{CAJ,randLUPI} can improve accuracy. \cite{xModal} transform V images into a new modality to reduce the gap with I images. \cite{wu2017rgb} conducted a pixel-to-pixel feature fusion operation on V and I images to build the synthetic images, and \cite{zhang2021towards} used a shallow auto-encoder to generate intermediate images from both modalities. While these generated images bridge the cross-modal gap, details of the image are lost. Accuracy is degraded in the context of multi-step learning. To address this, we extract body part prototypes and gradually use them to define multiple intermediate feature spaces.

\noindent \textbf{(c) Domain Generalization and Adaptation:} 
Data augmentation is an effective strategy to learn domain-invariant features and improve model generalization~\cite{DG_survey}.  
DG methods focus on manipulating the inputs to assist in learning general representations. \cite{zhou2020deep} relies on domain information to create additive noise that increases the diversity of training data while preserving its semantic information. \cite{mixup, DGzhou2021domain} uses Mixup to increase data diversity by blending examples from the different training distributions.


\edit{Some domain adaptation (DA) methods allow for bridging significant shifts between source and target data distributions. Gradual and multi-step (or transitive) methods~\cite{deepDA_survey} define the number and location of intermediate domains. 
Intermediate Domain Labeler~\cite{chen2021gradual} also creates multiple intermediate domains, where a coarse domain discriminator sorts them such that the cycle consistency of self-training is minimized.  
IDM\cite{IDM} mixes entire latent features to create an intermediate domain that lies on the shortest path between the target and source. In fact, by utilizing appropriate intermediate domains, the source knowledge can be better transferred to the target domain. By mixing the entire hidden features, this domain may lose the local ID-discriminative information and reduce the diversity of part features. This paper reduces the cross-modality gap by exchanging features between domains and creating intermediate virtual domains through gradual mixing during training.}

\vspace{-0.1cm}
 \section{Proposed Method}
\label{sec:proposed}
\vspace{-0.2cm}


Our BMDG approach relies on multiple virtual intermediate domains created from V and I images to learn a common representation space for V-I person ReID. 
Fig. \ref{fig:method}(a) shows the overall BMDG training architecture comprising two modules. (1) Our \textbf{part prototype alignment learning} module extracts semantically aligned and discriminative part prototypes from I and V modalities through hierarchical contrastive learning. Each prototype represents a specific part feature in a cropped person image. Exchanging aligned part prototypes allows for the gradual creation of ID-informative intermediate spaces. This module discovers and aligns part features that are distinct and informative about the person to transform the representation feature of one person between modalities robustly. (2) Our \textbf{bidirectional multi-step learning} module creates the auxiliary intermediate feature spaces at each step by progressively mixing more learned part prototypes from each modality. This gradually reduces the modality-specific information in the final feature representation.
In other words, two intermediate feature spaces are generated, at each step, by combining proportions of aligned sub-features from each modality. The model gradually learns to reduce modality-specific information and to leverage more common information from such intermediate domains by increasing the portions and making more complex training cases. BMDG trains the feature backbone by learning from easy samples with a lower modality gap to more complex ones with a higher gap. 
\begin{figure*}[ht]
  \centering
  \begin{subfigure}[t]{0.9\textwidth}
        \raisebox{-\height}{
        \includegraphics[width=\textwidth]{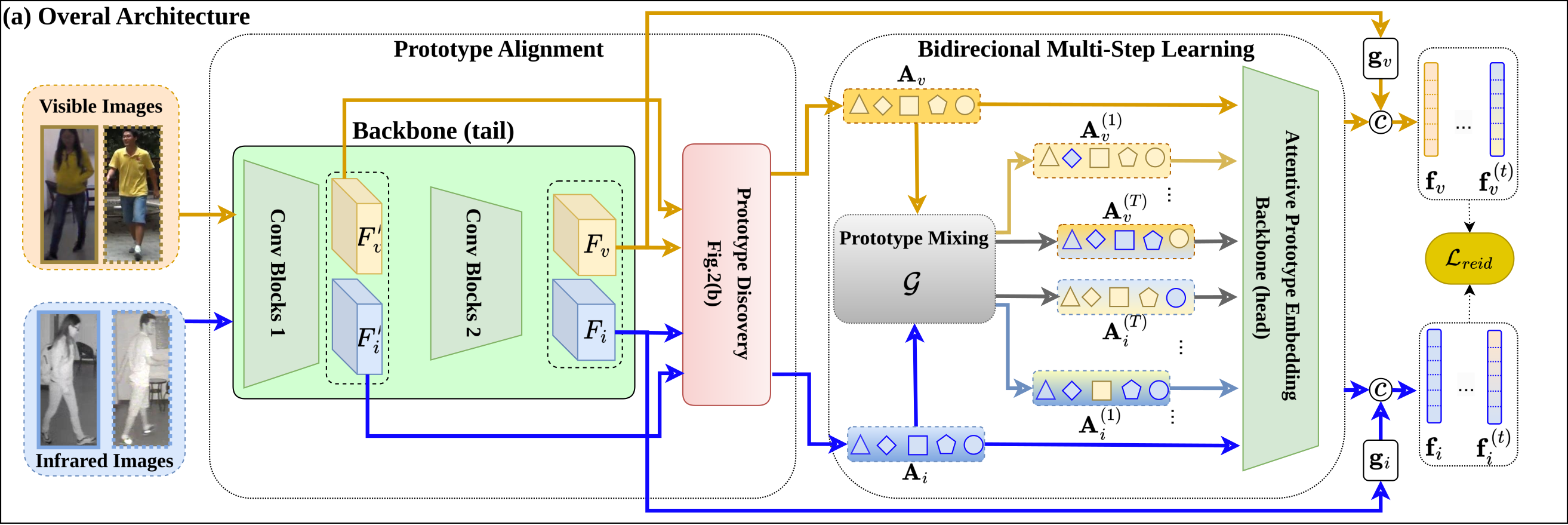}
        }
    \end{subfigure}
    \begin{subfigure}[t]{0.885\textwidth}
        \raisebox{-\height}{
        \includegraphics[width=\textwidth]{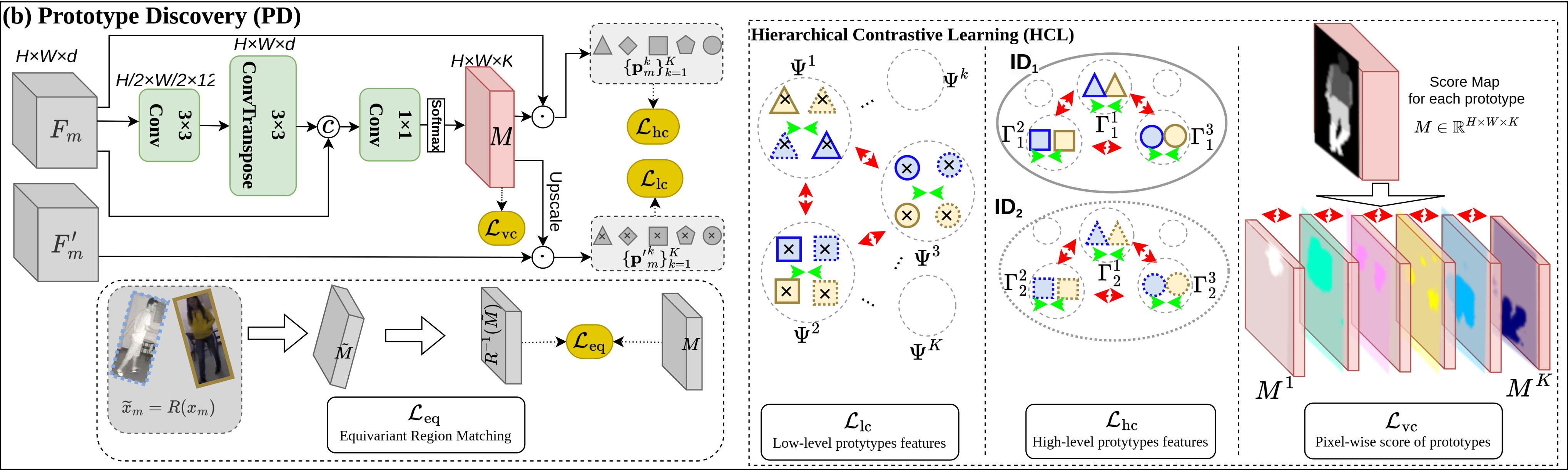}
        }
    \end{subfigure}
    \begin{subfigure}[t]{0.885\textwidth}
        \raisebox{-\height}{
        \includegraphics[width=\textwidth]{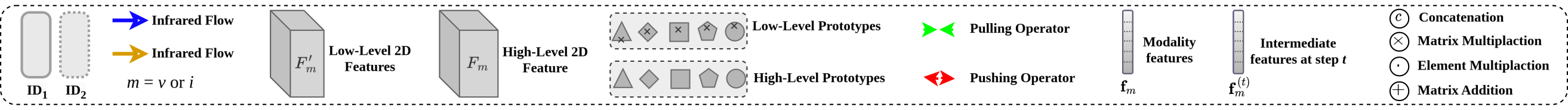}
        }
    \end{subfigure}
    \vspace{-.15cm}
  \caption{(a) Overall BMDG training architecture comprises two parts. The prototype alignment module (left) extracts body part prototype representations from V and I images. The bidirectional multi-step learning module (right) extracts discriminant features using multiple intermediate domains created by mixing prototype information. (b) Prototype discovery (PD) architecture mines prototypes from spatial features, and
  Hierarchical contrastive learning (HCL) encourages the prototypes to focus on similar semantics for all individuals without losing ID-discriminative information.}
  \label{fig:method}
  \vspace{-.2cm}
\end{figure*}

\noindent \textbf{Preliminaries:}  \label{sec:problem}
For the training process, let us assume a multimodal dataset for V-I ReID, represented by $\mathcal{V}=\{x^j_v, y_v^j\}_{j=1}^{N_v}$ and $\mathcal{I}=\{x^j_i,y_i^j\}_{j=1}^{N_i}$ which contains sets of V and I images from $C_y$ distinct individuals, with their ID labels. V-I ReID systems seek to match images captured from V and I cameras by training a deep backbone model to encode modality-invariant person embedding, denoted by $\mathbf{f}_v$ and $\mathbf{f}_i$. Given V and I images, the objective is to: 
\vspace{-0.1cm}
\begin{equation}
    \label{eq:problem}
    \max \{S(\mathbf{f}_v^n, \mathbf{f}_i^l) \cdot o \} 
    \text{ where } o\text{ = 1 if }  y_v^n = y_i^l \text{ else } o\text{ = -1}, 
    \vspace{-0.1cm}
\end{equation}
where $\mathbf{f}_v^n, \mathbf{f}_i^l \in \mathbb{R}^d$, $d$ is the dimension of the representation space, and $S(.,.)$ is similarity matching function.

\vspace{-0.15cm}
\subsection{Prototype Learning Module}
\vspace{-0.2cm}
To align persons' sub-features between different images in I and V, this module is used to discover different body-part attributes from diverse regions and learn to align these prototypes by making them complementary, interchangeable, and discriminative, as discussed in Section \ref{sec:intro}. 
\begin{align}
    \mathcal{H}\colon & \mathbb{R}^{W\times H \times 3} \to\mathbb{R}^{K \times d}\\
    & x \sim X  \mapsto \mathcal{H}(x) \sim P
\end{align}


\noindent \textbf{(a) Prototype Discovery (PD).} 
This module is proposed to extract the prototypes and discover several discriminative regions on the backbone feature maps. 
At first, the PD predicts a pixel-wise masking score (probability) $M \in \mathbb{R}^{H\cdot W \times K}$ for each prototype from the features map $\mathbf{F}_m \in \mathbb{R}^{H\cdot W \times d}$ where $H, W, d$ are the feature sizes, and $K$ is the total number of prototypes. $m$ is modality index which indicates $v$ as visible or $i$ as infrared domain. Then, it computes each prototype vector $\mathbf{p}_m^k$ by weighted aggregating of features in all pixels as:
\vspace{-0.15cm}
\begin{equation}
    \mathbf{p}_m^k = \frac{M_m^k}{|M_m^k|}\sum_{u \in U} [\mathbf{F}_m]_u \ , \; \; \;  |M_m^k| =  \sum_{u \in U}[M_m^k]_u
    \vspace{-0.1cm}
    \label{eq:pm}
\end{equation}
where for each pixel $u \in U = \{1,\dots,H\cdot W\}$, we have $\sum_{k=1}^{K}[M_m^k]_u=1$.
This module (illustrated in Fig. \ref{fig:method}(b) uses a shallow U-Net\cite{unet} architecture to create a region mask for each prototype. In the down-sample block, a convolution receptive field is applied on masks to make them lose locality by leveraging their neighbor pixels. At the same time, the up-sample propagates these high-level, coarsely localized masks into each original size. Skip connections are used to re-inject the local details lost in the down-sample phase. 


\noindent \textbf{(b) Hierarchical Contrastive Learning:}
To ensure the part prototypes represent the diverse attributes of each person, they must be non-redundant, and be independent of the others. This independence can be formulated that the mutual information ($MI$) between the distribution of random variables describing parts as $P^k$ should be zero as:
\vspace{-0.1cm}
\begin{equation}
    \text{MI}(P^k;P^q) = 0\; \forall k,q \in \{1,\dots,K\}, \; k \neq q,
    \label{eq:mi1}
    \vspace{-0.15cm}
\end{equation}
where $P^k$ is the distribution of the random variable of part $k$. 
Since the estimation of mutual information is complex and time-consuming, $MI(P^k,P^q)$ is minimized by directly reducing the cosine similarity between the extracted features of different parts as:
\vspace{-0.1cm}
\begin{equation}
    \min S \{(\mathbf{p}^k, \mathbf{p}^q)\} \; \forall k,q \in \{1,\dots,K\}, \; k \neq q,
    \vspace{-0.15cm}
\end{equation}
where $\mathbf{p}^k$ and $\mathbf{p}^q$ are features for part $k$ and $q$, respectively.
To ensure that the prototype features are semantically consistent around body parts that can be aligned in different samples, we maximize the similarity of $\mathbf{p}^k$ and all other prototypes with the same index in the training batch as $\mathbf{\hat{p}}^k$. 

To select samples in each batch with a positive prototype index, there are two options: the same parts belonging to the same or different persons (inter-person) as $\Psi^k$, and the same parts belonging to the same person with identity $y$ (intra-person) as $\Gamma^k_y$ ( see Fig. \ref{fig:method}.b). Since the goal is to extract dissimilar features in the representation space for two different persons, pulling apart the same prototypes from these two persons may disrupt this objective. 
Therefore, two contrastive losses as $\mathcal{L}_\text{lc}$ and $\mathcal{L}_\text{hc}$ are employed to provide inter-person information at a low-level and intra-person at a high-level:
\vspace{-0.25cm}
\begin{equation}
    \mathcal{L}_\text{lc} =- \sum_{k=1}^{K} \sum_{\mathbf{\hat{p}'} \in \Psi^k}\log \frac{e^{\mathbf{p}'^k\cdot\mathbf{\hat{p}'}/\tau}}{e^{\mathbf{p}'^k\cdot\mathbf{\hat{p}'}/\tau} + \sum_{q\neq k} e^{\mathbf{p'}^k\cdot\mathbf{p'}^q/\tau}},
    \label{eq:cons}
\end{equation}
\vspace{-0.30cm}
\begin{equation}
    \mathcal{L}_\text{hc} =- \sum_{y=1}^{C_y} \sum_{k=1}^{K} \sum_{\mathbf{\hat{p}} \in \Gamma^k_y}\log \frac{e^{\mathbf{p}^k\cdot\mathbf{\hat{p}}/\tau}}{e^{\mathbf{p}^k\cdot\mathbf{\hat{p}}/\tau} + \sum_{q\neq k} e^{\mathbf{p}^k\cdot\mathbf{p}^q/\tau}},
    \label{eq:cons2}
    \vspace{-0.15cm}
\end{equation}
where $\mathbf{p}_m'^k$ are low-level part prototype features vector computed same as Eq. (\ref{eq:pm}) but extracted from lower-level layers of the feature backbone.
Unlike contrastive learning introduced in \cite{PAT,choudhury2021unsupervised} that are used only for final features, the HCL optimizes at two levels as low-level and high-level features to force the model to cluster inputs based on parts and then cluster inside identities, as illustrated in Fig. \ref{fig:method}(b). For example, $\mathcal{L}_\text{lc}$ encourages the low-level prototype of legs for two persons to be similar and differ from other prototypes such as the torso. While $\mathcal{L}_\text{hc}$ encourages the high-level legs prototypes only differ from other prototypes in that person. 
Also, minimizing the distance between features of pixels belonging to the same part occurrence encourages the model to extract more similar features for each prototype:
\vspace{-0.25cm}
\begin{equation}
    \vspace{-0.15cm}
    \mathcal{L}_\text{c} =\sum_{k=1}^{K}\sum_{u \in U}[M_m^k\Vert \mathbf{p}_m^k -  \mathbf{F}_m\Vert]_u.
    \label{eq:part}
    \vspace{-0.15cm}
\end{equation}
Describing part occurrence to a single feature vector enables us to exchange these features for the corresponding prototype in different modalities and creates the intermediate step. 
Apart from making part features independent of each other by feature contrastive losses, 
for pushing prototypes to focus on different semantic body parts in different image locations, the visual contrastive is used as: 
\vspace{-0.15cm}
\begin{equation}
    \vspace{-0.15cm}
    \mathcal{L}_\text{vc} =\sum_{k=1}^{K}\sum_{q=k+1}^{K}\sum_{u \in U}[\Vert M^k_m -  M^q_m\Vert]_u,
    \label{eq:part2}
\end{equation}
to extract parts from diverse regions of backbone features with the minimum intersection.
\edit{
Also, for equivariant matching of regions, a random rigid transformation ($R$) is applied on the input images ($x_m$) to create the transformed images ($\Tilde{x}_m$). Both are then processed by the model, and the transformation on the masking score ($\Tilde{M}_m$) is inverted from the transformed images. The objective is to minimize the distance between the transformed and original maps:
\vspace{-0.15cm}
\begin{equation}
    \vspace{-0.15cm}
    \mathcal{L}_\text{eq} =\sum_{k=1}^{K}\Vert M^k_m -  R^{-1}(\Tilde{M}^k_m)\Vert.
    \label{eq:eq}
\end{equation}
}

\noindent \textbf{(c) Part Discrimination.} 
The information encoded by prototypes must describe the object in the foreground. That is, the mutual information between the joint probability $\{P^k\}_{k=1}^{K}$ and identity $Y$ should be maximized:
\vspace{-0.15cm}
\begin{equation}
    \text{MI}(P^1, \dots, P^{K};Y ) .
    \label{eq:mi2}
    \vspace{-.2cm}
\end{equation}
In the supplementary material, we proofed that for maximizing Eq. (\ref{eq:mi2}), we can minimize cross-entropy loss between each prototypes features and identities. In fact, by doing this, the model makes the part prototype regions lie on the object and have enough information about the foreground. In other words, each part prototype should be able to recover the identity of the person in the images.
So part ID-discriminative loss is defined as:
\vspace{-0.2cm}
\begin{equation}
    \mathcal{L}_\text{p} =\frac{-1}{K} \sum_{k=1}^{K} \sum_{c=1}^{C_y} y_c\log(W^k_c(\mathbf{p}_m^k)),
    \label{eq:pid}
    \vspace{-0.25cm}
\end{equation}
where $W^k$ is linear layer predicting the probability of identity $y_c$ from $\mathbf{p}_m^k$. \edit{To ensure prototype diversity and discrimination, we do not share parameters of $W^k$ and $W^q$ and randomly drop out a proportion of them during training, preventing any single part from dominating. }

\begin{table*}[!t]
 \vspace{-0.1 cm}
\centering
\vspace{-0.2cm}
\resizebox{\textwidth}{!}{
\begin{threeparttable}
\begin{tabular}{|c|l|c||c|c|c||c|c|c||c|c|c||c|c|c|} 
\hline
\multicolumn{3}{|c||}{\multirow{2}{*}{\textbf{Family}}}                     & \multicolumn{6}{c||}{\textbf{SYSU-MM01}}                                                                                                                                                                          & \multicolumn{6}{c|}{\textbf{RegDB}}                                                                                                                                                                                     \\ 
\cline{4-15}
\multicolumn{3}{|c||}{}                                                     & \multicolumn{3}{c||}{All Search}                                                                        & \multicolumn{3}{c||}{Indoor Search}                                                                     & \multicolumn{3}{c||}{Visible $\rightarrow$ Infrared}                                                          & \multicolumn{3}{c|}{Infrared $\rightarrow$ Visible}                                                     \\ 
\hline
\multicolumn{2}{|c|}{\textbf{Method}} & \textbf{Venue} & \textbf{R1}    & \textbf{R10}   & \textbf{mAP}   & \textbf{R1}    & \textbf{R10}  & \textbf{mAP}   & \textbf{R1}   & \textbf{R10}   & \textbf{mAP}  & \textbf{R1}    & \textbf{R10}   & \textbf{mAP}   \\  \hline \hline
\multirow{4}{*}{\rotatebox[origin=c]{90}{Global}} & AGW \cite{all-survey} & TPAMI'20                             & 47.50                                & -                          & 47.65                               & 54.17                                & -                          & 62.97                               & 70.05                               & -                                  & 50.19                              & 70.49                                & 87.21                              & 65.90                       \\
 &CAJ \cite{CAJ}      & ICCV'21                              & 69.88                                & -                          & 66.89                               & 76.26                                & 97.88                      & 80.37                               & 85.03                               & 95.49                              & 79.14                              & 84.75                                & 95.33                              & 77.82                       \\
 & RAPV-T \cite{zeng2023random}&ESA'23   & 63.97 & 95.30  & 62.33 & 69.00 & 97.39 & 75.41 & 86.81 & 95.81  & 81.02 & 86.60 & 96.14 & 80.52\\
 & G2DA \cite{WAN2023109150} & PR'23  &  63.94 & 93.34 & 60.73 & 71.06 & 97.31 & 76.01 & - & - & - & - & - & -   \\ 
 \hline
 
  \multirow{7}{*}{\rotatebox[origin=c]{90}{Part-based}} &DDAG \cite{DDAG}                                & ECCV'20                              & 54.74                                & 90.39                      & 53.02                               & 61.02                                & 94.06                      & 67.98                              & 69.34                               & 86.19                                  & 63.46                               & 68.06                                 & 85.15                                  & 90.31                        \\
  &SSFT  \cite{cmSSFT}                                & CVPR'20                              & 63.4                                 & 91.2                       & 62.0                                & 70.50                                & 94.90                      & 72.60                               & 71.0                                & -                                  & 71.7                               & -                                    & -                                  & -                           \\
  &MPANet \cite{wu2021Nuances}                                & CVPR'22                              & 70.58                                & 96.21                      & 68.24                               & 76.74                                & 98.21                      & 80.95                               & 83.70                               & -                                  & 80.9                               & 82.8                                 & -                                  & 80.7                        \\
  
  &SAAI \cite{SAAI}\tnote{a}              & ICCV'23          & 75.03            & -      & 71.69           & -            &        & -           & -           & -              & - & -            & -              & -   \\
  &SAAI \cite{SAAI} \tnote{b}              & ICCV'23          & 75.90            & -      & \underline{77.03}           & 82.20            &   -     & 80.01           & 91.07           & -              & \underline{91.45} & 92.09            & -              & \underline{92.01}   \\
 &CAL \cite{CATA}               & ICCV'23          & 74.66            & 96.47  & 71.73           & 79.69            & 98.93  & 83.68           & \underline{94.51}  & \textbf{99.70} & 88.67          & \underline{93.64}    & \textbf{99.46} & 87.61   \\
 &PartMix \cite{PartMix23}               & CVPR'23          & \underline{77.78}            & -  & 74.62           & 81.52            & -  & 84.83           & 84.93  & - & 82.52          & 85.66    & - & 82.27   \\
 \hline

\multirow{6}{*}{\rotatebox[origin=c]{90}{Intermediate}}  & SMCL\cite{wei2021syncretic}& ICCV'21& 67.39 & 92.84 & 61.78  & 68.84 & 96.55 & 75.56& 83.93 & -  & 79.83 & 83.05 & - & 78.57 \\
 & MMN \cite{zhang2021towards}&ICM'21  & 70.60 & 96.20  & 66.90 & 76.20 & \textbf{99.30} & 79.60& 91.60 & 97.70  & 84.10 & 87.50 & 96.00 & 80.50  \\ 
 & RPIG \cite{randLUPI}&ECCVw'22  & 71.08 & 96.42  & 67.56 & 82.35 & 98.30 & 82.73 & 87.95 & 98.3 & 82.73 & 86.80 & 96.02 & 81.26\\
 & FTMI \cite{sun2023visible}&MVA'23 & 60.5 & 90.5 & 57.3  & - & - & - & 79.00 & 91.10 & 73.60  & 78.8 & 91.3 & 73.7  \\ 
 
 &G2DA        \cite{WAN2023109150}                          & PR'23                                & 63.94                                & 93.34                      & 60.73                               & 71.06                                & 97.31                      & 76.01                               & -                                   & -                                  & -                                  & -                                    & -                                  & -                           \\

 &SEFL \cite{shape-Erase23}                                  & CVPR'23         & 75.18           & 96.87 & 70.12          & 78.40           & 97.46 & 81.20          & 91.07          & -             & 85.23         & 92.18           & -             & 86.59  \\
  \hline
& \multicolumn{1}{|c|}{BMDG (ours)\tnote{a} \text{ }} & -              &{76.15} & \underline{97.42}             & {73.21} & \underline{83.53} & 98.02                      & \underline{84.92} & 93.92 & 98.11                      & 89.18 & \underline{94.08} & 97.0                      & 88.67              \\

& \multicolumn{1}{|c|}{BMDG (ours)\tnote{b} \text{ }} & -              & \textbf{78.08} & \textbf{97.90}             & \textbf{78.22} & \textbf{83.59} & \underline{98.96}                      & \textbf{86.35} & \textbf{94.76} & \underline{98.91}                      & \textbf{92.21} & \textbf{94.56} & \underline{ 98.31}                      & \textbf{93.07}              \\
\hline
\end{tabular}
\footnotesize
   \begin{tablenotes}
        \item[a] without AIM  
        \item[b] with AIM 
      \end{tablenotes}
\vspace{-0.2cm}
\caption{Accuracy of the proposed BMDG and state-of-the-art methods on the SYSU-MM01 (single-shot setting) and RegDB datasets. All numbers are percent. Results for methods were obtained from the original papers. AIM \cite{SAAI} means re-ranking method.}
\label{tab:all-results}
\end{threeparttable}

}
 \vspace{-0.25 cm}
\end{table*}

\subsection{Bidirectional Multi-step Learning}
\vspace{-0.15cm}
After identifying prototypes, multiple auxiliary intermediate domains are created, and features are extracted using the Attentive Prototype Embedding (APE) module.

\noindent \textbf{(a) Attentive Prototype Embedding.}  
Instead of directly concatenating learned prototypes as the final feature, we process all other prototypes in input images. The motivation of this design is that each prototype may discriminate the part-attributes shared for all individuals rather than ID attributes. 
The Attentive Prototype Embedding module (which is detailed in the supply. materials), $\mathcal{F}$, leverages relevant person information between prototypes by applying an attention mechanism to achieve person-aware final features. To aggregate the ID-discriminative features of prototypes, the APE module first uses a fully connected layer to score and emphasize important channels in their feature map. It weights each prototype feature by computing the similarities between them. 
\begin{equation}
\begin{gathered}
    \mathcal{F}(\mathbf{A}_m) = \mathbf{W}_{\text{mlp}}(\mathbf{C}_m),   \text{ where }  \mathbf{C}_m = \mathbf{W}_v(\mathbf{A}_m) \otimes \mathbf{B}_m\text{, } \\ \mathbf{B}_m = \sigma(\mathbf{W}_q(\mathbf{A}_m) \otimes \mathbf{W}_k(\mathbf{A}_m)),
    \vspace{-0.25cm}
    \end{gathered}
\end{equation}
and 
\begin{equation}
\mathbf{A}_m=[\mathbf{p}_m^1;\mathbf{p}_m^2;\dots;\mathbf{p}_m^{K}],
\vspace{-0.1cm}
\end{equation}
where $\mathbf{A}_m \in \mathbb{R}^{K\times d}$, $\otimes$ is matrix multiplication and $\sigma$ is the Sigmund function. $\mathbf{W}_{\text{mlp}}$, $\mathbf{W}_v$, $\mathbf{W}_q$ and $\mathbf{W}_k$ are linear layer.  
Also, to increase the discriminative ability of the final feature embedding, the global features, $\mathbf{g}_m$, extracted by the backbone head are concatenated to the output of APE:
\vspace{-0.15cm}
\begin{equation}
    \mathbf{f}_m = [\mathcal{F}(\mathbf{A}_m); \mathbf{g}_m],
    \label{eq:f}
    \vspace{-0.1cm}
\end{equation}
where $\mathbf{g}_m$ is denoted by $\frac{1}{H\cdot W}\sum_{u \in U}[\mathbf{F}_m]_u$.

Our APE has two advantages: (1) it allows adjusting modality-agnostic attention between prototypes regardless of the modality by applying cross-modality prototype features, and (2) it improves the representation ability of final embedding by emphasizing most discriminative prototypes.




\noindent \textbf{(b) Bidirectional Multi-Step Learning.}  
To deal with the significant shift between modalities in the feature space, the model is trained via multiple intermediate steps that gradually bridge this domain gap. Using intermediates with less domain shift, the model gradually learns to leverage cross-modality discriminating clues at each step\cite{wang2022understanding,chen2021gradual,abnar2021gradual,zhang2021gradual}. Initially, the discrepancies between modalities are small, letting the model learn from the easier samples first, then converge on more complex cases with larger shifts. One solution to achieve these intermediates is to start from one modality as the source and transform gradually to the other as the target. However, this makes the model forget the learned knowledge of the source and be biased on the target.

To address this issue, modalities are transformed bidirectionally to create an intermediate domain at each step. By gradually transforming domains to each other, the domain gap vanishes over multiple steps. Each intermediate auxiliary step provides discriminative information about the person across both modalities, enabling the training process to transform inputs from one to the other. By creating these intermediate domains, our model aims to neither lose nor duplicate information from the main domains. For gradual transformation in each step, we exchange the features of the same prototypes from both modalities to create a mixed and virtual representation space with less domain discrepancies. These prototype-mixed intermediate feature spaces are used alongside V and I prototypes to provide APE along with the ability to extract common features from the modalities.
For step $t$, which $t\in \{1, \dots, T\}$ and $T$ is the number of intermediate steps, the intermediate features for the same person are generated using a mixing function $\mathcal{G}(.,.,.)$ as:
\vspace{-0.1cm}
\begin{equation}
    \mathbf{A}^{(t)}_m = \mathcal{G}(\mathbf{A}_m, \mathbf{A}_{\Tilde{m}},t),
\end{equation}
where $\Tilde{m} \neq m$ which mixes the prototypes from two modalities. For example, we use a simple but yet effective random prototype mixing strategy for $\mathcal{G}(.,.,.)$: 
\begin{equation}
    \mathcal{G}(\mathbf{A}_m, \mathbf{A}_{\Tilde{m}},t) =[\mathbf{p}_{r(m,\Tilde{m},t)}^1,\dots,\mathbf{p}_{r(m,\Tilde{m},t)}^K],
    \vspace{-0.1cm}
\end{equation}
where $r(.,.,.)$ is a weighted random selector:
\vspace{-0.2cm}
\begin{equation}
  r(m,\Tilde{m},t) =   \begin{cases}
  m  & t/T \leq \mathcal{U}(0,1) \\
  \Tilde{m} & \text{ else}
\end{cases}
\vspace{-0.2cm}
\end{equation}
where $\mathcal{U}(0,1)$ is a uniform random generator.  
For intermediate step ($t<T$), we have $\mathbf{f}^{(t)}_m = [\mathcal{F}(\mathbf{A}^{(t)}_m); \mathbf{g}_m]$ and in the last step   $\mathbf{f}^{(T)}_m = [\mathcal{F}(\mathbf{A}^{(T)}_m); \mathbf{g}_{\Tilde{m}}]$ is used.
The module $\mathcal{F}$ learns to gradually reduce the discrepancies in representation space by applying metric learning objectives bidirectionally between $\mathbf{f}_i$ and $\mathbf{f}^{(t)}_i$ and between $\mathbf{f}_v$ and $\mathbf{f}^{(t)}_v$ as :
\vspace{-0.2cm}
\begin{align}
\label{eqn:eqlabel}
\begin{split}
\mathcal{L}_{\text{bcc}} &= \mathcal{L}_{\text{cc}}(\mathbf{f}_v,\mathbf{f}^{(t)}_v) + \mathcal{L}_{\text{cc}}(\mathbf{f}_i,\mathbf{f}^{(t)}_i)
\\
 \mathcal{L}_{\text{bce}} &= \mathcal{L}_{\text{ce}}(\mathbf{f}_v) +\mathcal{L}_{\text{ce}}(\mathbf{f}^{(t)}_v) + \mathcal{L}_{\text{ce}}(\mathbf{f}_i) + \mathcal{L}_{\text{ce}}(\mathbf{f}^{(t)}_i)
 \\
 \mathcal{L}_{\text{re}} &= \mathcal{L}_{\text{bce}} + \mathcal{L}_{\text{bcc}},
\end{split}
\vspace{-0.2cm}
\end{align}
that $\mathcal{L}_{\text{cc}}$, $\mathcal{L}_{\text{ce}}$ are center cluster and cross-entropy losses \cite{wu2021Nuances}. 
Since $\mathbf{f}^{(t)}_i$ begins from $\mathbf{f}_i$ at first steps and approaches to $\mathbf{f}_v$ at ending steps, the model is gradually trained to extract more robust modality invariant features by seeing samples with low modality gap to the harder ones with higher gap\cite{wang2022understanding}.

\vspace{-0.1cm}
\subsection{Training and Inference}
\vspace{-0.15cm}
\edit{In each step, the model tries to represent the same feature space for two input domains, V/I and the intermediate ones, using the overall loss expressed as:
\vspace{-0.2cm}
\begin{equation}
\label{eq:all_losses}
    \mathcal{L} = \mathcal{L}_{\text{re}} + \lambda_{\text{f}}(\mathcal{L}_{\text{lc}}+\mathcal{L}_{\text{hc}}) + \lambda_{\text{v}}\mathcal{L}_{\text{vc}} + \lambda_{\text{c}}\mathcal{L}_{\text{c}} + \lambda_{\text{i}}\mathcal{L}_{\text{p}} + 
    \lambda_{\text{e}}\mathcal{L}_{\text{eq}}
    \vspace{-0.1cm}
\end{equation}
where $\lambda$ are hyper-parameters for weighting losses.
During inference, the tail and head backbones are combined to extract the prototypes and global features as $\mathbf{A}_i$, $\mathbf{A}_v$, $\mathbf{g}_i$ and $\mathbf{g}_v$ from given input image $x_i$ and $x_v$. Using Eq. (\ref{eq:f}), the $\mathbf{f}_i$ and $\mathbf{f}_v$ are computed by the APE. The matching score for a query input image is computed by cosine similarly.}
\vspace{-0.2cm}
 \section{Results and Discussion}
\vspace{-0.1cm}

In this section, we compare the BMDG with SOTA V-I ReID methods, including global, part-based, and intermediate approaches. Extensive ablation studies are conducted to show the effectiveness of the proposed bi-directional learning, the impact of the number of parts-prototypes, and intermediate steps during learning. The supplementary materials provide detailed information on our experimental methodology, including the implementation of baseline models, performance measures, and descriptions of the SYSU-MM01 \cite{SYSU}, RegDB \cite{regDB}, and LLCM \cite{LLCM} datasets.



\begin{table}[!b]
\centering
\vspace{-0.3cm}
\resizebox{0.99\linewidth}{!}{%
\begin{tabular}{|l|l||c|c||c|c|} 
\hline
\multicolumn{2}{|c||}{\multirow{2}{*}{\textbf{Family}}} & \multicolumn{4}{c|}{\textbf{LLCM}}  \\
\cline{3-6}
\multicolumn{2}{|l||}{}& \multicolumn{2}{c||}{V $\rightarrow$ I}  & \multicolumn{2}{|c|}{I $\rightarrow$ V} \\ \hline
\multicolumn{1}{|l|}{\textbf{Method}} & \textbf{Venue}   & \textbf{R1} & \textbf{mAP} & \textbf{R1}  & \textbf{mAP}\\ \hline \hline
DDAG\cite{DDAG} & ECCV'20 & 40.3 & 48.4 & 48.0 & 52.3 \\
CAJ\cite{CAJ} & ICCV'21 & 56.5 & 59.8 & 48.8 & 56.6 \\
RPIG\cite{randLUPI} & ECCVw'22 & 57.8 & 61.1 & 50.5 & 58.2 \\
DEEN\cite{DEEN} & CVPR'23 & 62.5 & 65.8 & 54.9 & 62.9 \\
\hline 

BMDG (ours) & - & \textbf{63.4}  & \textbf{66.3} & \textbf{56.4} & \textbf{63.5} \\
\hline
\end{tabular}}
\vspace{-0.15cm}
\caption{
Accuracy of the proposed BMDG and state-of-the-art methods on the Large LLCM Dataset. All numbers are percent.}
\label{tab:LLCM-results}
\end{table}

\vspace{-0.1cm}
\subsection{Comparison with State-of-Art Methods:}
\vspace{-0.2cm}
Table \ref{tab:all-results} compares the accuracy of BMDG with state-of-the-art V-I ReID approaches. Our experiments on the SYSU-MM01 and RegDB datasets show that BMDG outperforms the SOTA methods in the majority of situations. Compared to \cite{PartMix23}, BMDG has lower R1 accuracy (-1\%) in the All Search settings on the SYSU-MM01 dataset. However, it achieves higher performance in Indoor settings (+2\%) and on the RegDB dataset (+10\%). These results can be attributed to BMDG's ability to effectively capture more ID-related knowledge across modalities. This is achieved by learning discriminative part-prototypes and gradually reducing modality-specific information in the extracted final features. Additionally, the prototype alignment module creates gradual and bidirectional intermediate spaces without sacrificing discriminative ability, allowing BMDG to surpass intermediate methods. Moreover, BMDG can be easily integrated into different part-based V-I ReID models, enhancing generality through gradual training. To show its effectiveness, we integrated our approach into state-of-the-art part-based or prototype-based baseline models \cite{DDAG,wu2021Nuances,SAAI}. We executed the original code published by authors using the hyper-parameters and other configurations they provided with and without BMDG. Table \ref{tab:baselines} shows that BMDG  improves performance for mAP and rank-1 accuracy for all methods by an average of 2.02\% and 1.54 \%, respectively.
\edit{
Similar performance improvements are observed on the large and complex LLCM dataset (see Table \ref{tab:LLCM-results}). Since LLCM was introduced recently, few papers reported their results, so we executed other approaches with published code on the LLCM dataset. 
}

\begin{table}[!t]
\small
\centering
\vspace{-0.3cm}
\resizebox{0.9\linewidth}{!}{
\begin{tabular}{|l||l|l|}
\hline
\textbf{Method}         & \textbf{R1 (\%)} & \textbf{mAP (\%)} \\ \hline \hline
DDAG  \cite{DDAG}               & 53.62           & 52.71            \\
DDAG with BMDG                  & \textbf{55.36} (\textcolor{green}{+1.74$\uparrow$})          & \textbf{54.05} (\textcolor{green}{+1.34$\uparrow$})           \\ \hline
MPANet \cite{wu2021Nuances}     & 66.24           & 62.89            \\
MPANet with BMDG                & \textbf{68.74} (\textcolor{green}{+2.50$\uparrow$})          & \textbf{64.25} (\textcolor{green}{+1.36$\uparrow$})           \\ \hline
SAAI \cite{SAAI}                & 71.87           & 68.16            \\
SAAI with BMDG                  & \textbf{73.69} (\textcolor{green}{+1.82$\uparrow$})           & \textbf{70.08} (\textcolor{green}{+1.92$\uparrow$})           \\ \hline
\end{tabular}}
 \vspace{-0.25 cm}
 \caption{Accuracy of part-based ReID methods with BMDG on the SYSU-MM01, under single-shot setting. Results were obtained by executing the author's code on our servers.}
\label{tab:baselines}
 \vspace{-0.25 cm}
\end{table}

\vspace{-0.1cm}
\subsection{Ablation Studies:} \label{sec:Ablation}
\vspace{-0.15cm}
\noindent \textbf{(a) Step size and number of part prototypes.} 
We evaluate the best option for the number of parts and steps. Table \ref{tab:part_step} (top) shows that using the BMDG approach increases the R1 accuracy when increasing the number of steps $T$ for a specific number of prototypes $K$. Also, it increases performance in multi-step with $K$ parts, letting the model learn from the lower gap sample and converge on harder cases. 
%

\begin{table}[!b]
    \vspace{-0.1 cm}
      \centering
            \resizebox{0.8\linewidth}{!}{
            \begin{tabular}{|c||cccccc|}
            \hline
            \multirow{2}{*}{$T$} & \multicolumn{6}{c|}{\textbf{Number of part prototypes ($K$)}} \\ \cline{2-7}
                &3     & 4     & 5     & 6              & 7     & 10    \\ \hline \hline
            \multicolumn{7}{|c|}{\textbf{Prototype exchanging (Our)}}  \\ \hline
            0  & 68.25 & 69.24 & 69.40 & 70.27          & 70.03 & 68.12 \\
            1  & 69.97 & 70.81 & 72.07 & 71.97          & 71.31 & 69.28 \\
            2  & 71.20 & 72.35 & 73.98 & 73.61          & 72.45 & 71.25 \\
            3  & 73.32 & 73.94 & 74.11 & 74.98          & 73.22 & 71.67 \\
            4  & -     & 74.08 & 74.15 & \textbf{75.43} & 73.51 & 71.99 \\
            6  & -     & -     & -     & 75.37          & 73.52 & 72.15 \\
            10 & -     & -     & -     & -              & -     & 72.07 \\ \hline \hline
            \multicolumn{7}{|c|}{\textbf{Entire mixup (IDM \cite{IDM})}}  \\ \hline 
             0      & 68.25 & 69.24 & 69.4      & 70.27          & 70.03 & 68.12 \\
             1      & 68.53 & 69.19 & 69.56     & 70.53          & 70.3 & 68.41 \\
             4      & 68.70 & 70.03 & 69.75     & 70.84 & 70.77  & 68.53 \\
             6      & 69.13 & 70.81 & 70.1      & \textbf{71.16}  & 70.29 & 68.90 \\
             10     & 69.21 & 70.15 & 69.52     & 70.95           & 70.90 & 68.13 \\ \hline
            \end{tabular}
            }
    \vspace{-0.15cm}
    \caption{R1 accuracy of BMDG using (a) our prototype mixing and (b) IDM \cite{IDM} for different numbers of part prototypes ($K$) and intermediate steps ($T$). mAP is reported in suppl. materials.}
    \label{tab:part_step}
    \vspace{-.15cm}
\end{table}

Another option for each step, instead of exchanging some parts, is to use the Mixup\cite{mixup} or IDM \cite{IDM}  strategy for all parts from two modalities with normalized weight: $\mathbf{A}^{(t)}_m = \alpha^{(t)} \mathbf{A}_m + (1-\alpha^{(t)}) \mathbf{A}_{m'}$, 
where $\alpha$ starts from 0 to 1 w.r.t. step $t$ and number of steps. Results in Table \ref{tab:part_step} (bottom) show that using multi-step increases performance marginally since it does not utilize the ability of prototypes to create meaningful intermediate steps. This is explained by the fact that prototypes contain some modality-specific information, which is crucial for transforming domains in intermediate steps and may be lost in the averaging process. 



\noindent {\bf (c) Loss functions.}
Contrastive objectives between parts ($\mathcal{L}_{\text{lc}}$ and $\mathcal{L}_{\text{hc}}$) are essential for the model to detect similar semantics between objects, thereby creating an intermediate step by swiping corresponding parts. Table \ref{table:losses} shows this necessity when used in conjunction with $\mathcal{L}_\text{c}$. Using such regularization in the multi-step approach increases mAP by 3\% and R1 by 4\%. 
Also, results of part separation loss, $\mathcal{L}_\text{vc}$, indicate that using this loss pushes the parts to be scattered on the body regions and gives the part features a more robust representation. Additionally, it enables the intermediate steps to generalize the feature space better by increasing the performance over the single-step with multi-step setting. 
The part identity loss $\mathcal{L}_{\text{p}}$ lets the model extract ID-related prototype features that contain discriminative information about the object and generate an id-informative intermediate step. Without this objective, parts features are likely to focus on background information. As shown in the two last rows, performance increases when using this objective.
\begin{table}[!t]
\small
\centering
\vspace{-0.1cm}
\resizebox{\linewidth}{!}{
\begin{tabular}{|c|c|c|c|c|c||c|c||c|c|}
\hline
\multicolumn{6}{|c||}{\textbf{Losses}}  & \multicolumn{2}{c||}{\textbf{R1 (\%)}} & \multicolumn{2}{c|}{\textbf{mAP (\%)}} \\ \hline
 $\mathcal{L}_{\text{p}}$ & $\mathcal{L}_{\text{lc}}$ & $\mathcal{L}_{\text{hc}}$ & $\mathcal{L}_{\text{vc}}$ & $\mathcal{L}_{\text{c}}$ & $\mathcal{L}_{\text{eq}}$ & $T=0$     & $T=K$ & $T=0$      & $T=K$ \\ \hline \hline
\multicolumn{6}{|c||}{baseline}         & 51.09   & - & 49.68    & - \\ \hline
              &    \cmark                        &  \cmark                          &                     &                    &   & 62.19   & 65.28 & 58.91    & 60.73 \\ 
             &    \cmark                        &  \cmark                          &     \cmark                &                  &     & 66.35   & 69.50 & 63.01    & 66.48 \\ 

                           &                            &                            & \cmark                    & \cmark                      & & 66.67   & 67.31 & 61.04    & 61.68 \\ 
 \cmark                     &                            &                            &                           & \cmark                     & & 59.78   & 60.66 & 55.41    & 56.05 \\ 
 \cmark                     & \cmark                     & \cmark                     & \cmark                    &                            & & 68.30   & 69.04 & 66.87    & 66.99 \\ 
& \cmark                     & \cmark                     & \cmark                    &\cmark                             & & 69.68   & 71.20 & 67.08    & 68.13 \\ 
\cmark & \cmark                     & \cmark                     & \cmark                    &\cmark                             & & 69.90   & 72.67 & 68.11    & 69.75 \\ 
 \cmark                     & \cmark                     & \cmark                     & \cmark                    & \cmark                   &  \cmark & \textbf{70.27}   & \textbf{75.37} & \textbf{68.45}    & \textbf{72.71} \\ \hline
\end{tabular}
}
\vspace{-.2cm}
\caption{Impact on the accuracy of using different BMDG losses, using SYSU-MM01 for single-step ($T=0$) and multi-step ($T=K$) settings. The baseline $\mathcal{L}_{\text{re}}$ loss is used in all cases.}
\label{table:losses}
\end{table}

\noindent {\bf (d) Bidirectional domain generalization.}
To show the impact of the BMGD approach, we conduct experiments under the following settings: (1) single-step, (2) one-directional multi-step (from V $\rightarrow$ I and I $\rightarrow$ V), and (3) bidirectional for creating the intermediate series. The results are shown in Table~\ref{table:bidirectional}. Although using multiple intermediate steps from one modality to another improves the model's accuracy by allowing it to gradually learn and utilize common discriminative features, this process can lead to losing information about the source modality. The bi-directional approach addresses these issues and achieves the best performance, as it is not biased towards any specific modality.
\begin{table}[h!]
\centering
\vspace{-0.1cm}
\resizebox{\linewidth}{!}{
\begin{NiceTabular}{|c||cc|cc|cc|cc|}
\hline
\multirow{3}{*}{\textbf{Settings}} & \multicolumn{8}{c|}{\textbf{Number of part prototypes ($K$)}}  \\ \cline{2-9}
                & \multicolumn{2}{c|}{4}     & \multicolumn{2}{c|}{5}     & \multicolumn{2}{c|}{6}     & \multicolumn{2}{c|}{7}     \\ \cline{2-9}
                & R1 & mAP & R1 & mAP   & R1 & mAP  & R1 & mAP \\ \hline \hline
Single step              & 69.2 & 65.9 & 69.4 & 66.1  & 70.2 & 66.3 & 70.4& 66.5\\ 
One (V $\rightarrow$ I)   & 71.0 & 66.5 & 71.5& 67.1 & 72.3& 67.5 & 71.6& 66.8 \\ 
One (I $\rightarrow$ V)   & 70.8& 66.3 & 71.2& 67.0 & 72.6& 67.6  & 71.0& 66.7  \\ 
Bidirectional       & \textbf{74.0} & \textbf{71.8} & \textbf{74.1}& \textbf{71.9} & \textbf{75.4}& \textbf{72.8} & \textbf{73.5}& \textbf{70.2} \\ \hline
\end{NiceTabular}
\smallskip\footnotesize
}
\vspace{-.15cm}
\caption{Impact on rank-1 and mAP accuracy using different settings for multi-step domain generalization. All numbers are percent. $V \rightarrow I$ means that the intermediates are created only from V, by replacing prototypes from I features.}
\label{table:bidirectional}
\vspace{-0.5cm}
\end{table}

\vspace{-0.1cm}
\section{Conclusion}
\vspace{-0.2cm}

This paper introduces the BMDG framework for V-I ReID. It uses a novel prototype learning approach to learn the most discriminative and complementary set of prototypes from both modalities. These prototypes generate multiple intermediate domains between modalities by progressively mixing prototypes from each modality, reducing the domain gap. The effectiveness of BMDG is shown experimentally on several datasets, where it outperforms state-of-the-art methods for V-I person ReID. Results also highlight the impact of considering multiple intermediate steps and bidirectional training to improve accuracy. Moreover, BMDG can be integrated into other part-based V-I ReID methods, significantly improving their accuracy. A limitation of BMDG is its dependency on discriminant body parts.



\noindent\textbf{Acknowledgements.} This research was supported by the Natural Sciences and Eng.  Research Council of Canada.   

{\small
\bibliographystyle{ieee_fullname}
\bibliography{main}
}
\begin{appendices}
\newtheorem{theorem}{Theorem}
\newtheorem{proof}{Proof}

\newtheorem{prop}{Proposition}


\appendix
\numberwithin{equation}{section}
\numberwithin{figure}{section}
\numberwithin{table}{section}
\newcommand{\I}{\text{MI}}
\setcounter{page}{1}
\onecolumn

\title{Bidirectional Multi-Step Domain Generalization\\ for Visible-Infrared Person Re-Identification: Supplementary} 
\maketitle

\section{Proofs} 
\label{sec:proofs}

In Section \ref{sec:proposed} of the manuscript, our model seeks to represent input images using discriminative prototypes that are complementary. To ensure their complementarity, we minimize the mutual information (MI) between the data distributions of each pair of prototypes using a contrastive loss between feature representations of the different prototypes. To improve the discrimination, the prototype representations encode ID-related information by maximizing the MI between the joint distribution among all prototypes in the label distribution space. In this section, we prove that by minimizing the cross-entropy loss between features of each prototype class and the person ID in images, we can learn discriminative prototype representations. 

\subsection{Maximizing $\I(P^1,\dots,P^K;Y)$}

This section provides a proof that $ \I(P^1,\dots,P^K;Y)$ (Eq. \ref{eq:mi2}) can be lower-bounded by :
\begin{equation}
  \I(P^1;Y) + \dots + \I (P^K;Y),
\end{equation}
following the properties of mutual information.

\noindent{\bf P.1 (Nonnegativity)}
For every pair of random variables $X$ and $Y$:
\begin{equation}
    \I(X;Y) \geq 0
\end{equation}
\noindent{\bf P.2 }
For every pair random variables $X$, $Y$ that are independent:
\begin{equation}
    \I(X;Y) = 0.
\end{equation}

\noindent{\bf P.3 (Monotonicity)}
For every three random variables $X$, $Y$ and $Z$:
\begin{equation}
    \I(X;Y;Z) \leq \I(X;Y)
\end{equation}
\noindent{\bf P.4 }
For every three random variables $X$, $Y$ and $Z$, the mutual information of joint distortions $X$ and $Z$ to $Y$ is:
\begin{equation}
    \I(X,Z;Y) = \I(X;Y) + \I(Z;Y) - \I(X;Z;Y)
\end{equation}

\begin{equation}
    MI(P^1, \dots, P^{K};Y ) .
    \label{eq:mi3}
    \vspace{-.2cm}
\end{equation}
\begin{figure}[hb!]
    \centering
    \includegraphics[width=\linewidth]{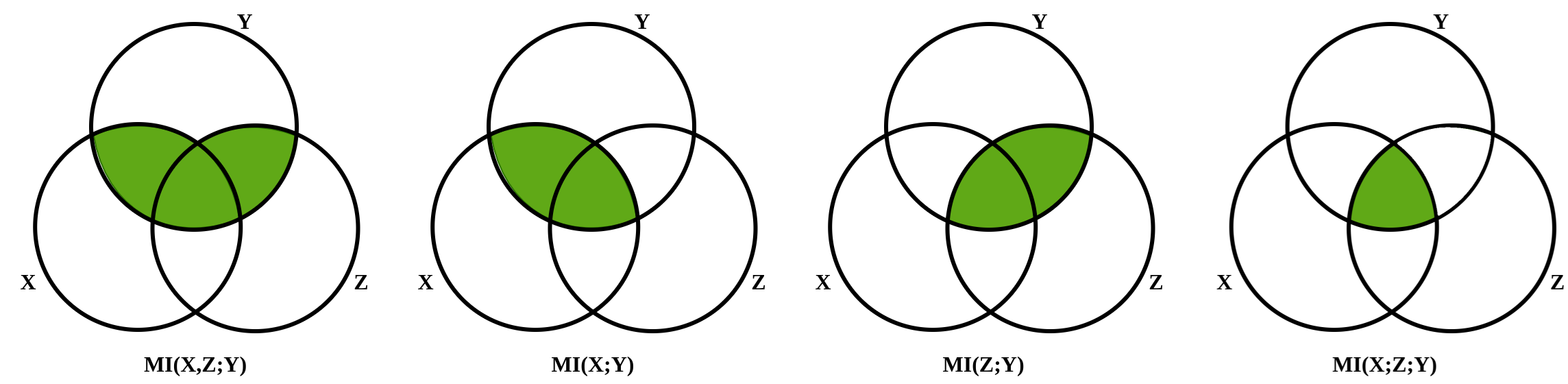}
    \caption{Venn diagram of theoretic measures for three variables $X$, $Y$, and $Z$, represented by the lower left, upper, and lower right circles, respectively.}
    \label{fig:enter-label}
\end{figure}

\begin{theorem}
Let $P^1,\dots, P^K$ and $Y$ be random variables with domains $\mathcal{P}^1,\dots,\mathcal{P}^K$ and $\mathcal{Y}$, respectively. Let every pair $P^k$ and $P^q$ ($k \neq q$) be independent. Then, maximizing $\I(P^1,\dots,P^K;Y)$ can be approximated by maximizing the sum of MI between each of $P^k$ to $Y$, $\sum_{k=1}^{K}{\I(P^k;Y)}$.
\end{theorem}
\begin{proof}
First, we define $\Tilde{P}^{k}$ as the joint distribution of $P^{k+1},\dots,P^{K}$. Using the {\bf P.4} we have:
\begin{equation}
    \I(P^1, \Tilde{P}^{1};Y) = 
    \underbrace{\I(P^1 ;Y)}_{\alpha}  + \underbrace{\I(\Tilde{P}^{1};Y)}_{\beta} - \underbrace{\I(P^1;\Tilde{P}^{1};Y)}_{\gamma},
    \label{eq:proof1}
\end{equation}
To maximize this, $\alpha$ and $\beta$ should be maximized and $\gamma$ minimized. Given {\bf P.3}, $\min \I(P^1;\Tilde{P}^{1};Y)$ can be upper-bounded by $\min \I(P^1;\Tilde{P}^{1})$:
\begin{equation}
    \label{eq:temp}
    \I(P^1;\Tilde{P}^{1};Y) \leq   \I(P^1;\Tilde{P}^{1}),
\end{equation}
So by minimizing the right term of Eq. \ref{eq:temp}, $\gamma$ is also minimized. We show that $\I(P^1;\Tilde{P}^{1})=0$ by expanding $\Tilde{P}^1$ to $(P^2, \Tilde{P}^2)$:
\begin{equation}
    \I(P^1; P^2, \Tilde{P}^{2}) = 
    \I(P^1 ;P^2) + \I(P^1; \Tilde{P}^{2}) - \I(P^1;P^2;\Tilde{P}^{2}).
\end{equation}
Given {\bf P.1} and Eq. \ref{eq:mi1}, $\I(P^1 ;P^2) = 0$ and $\I(P^1;P^2;\Tilde{P}^{2}) \leq \I(P^1;P^2) = 0$ so that:
\begin{equation}
    \I(P^1; P^2, \Tilde{P}^{2}) = \I(P^1; \Tilde{P}^{2}).
\end{equation}
After recursively expanding $\Tilde{P}^{2}$:
\begin{equation}
    \I(P^1; P^2, \Tilde{P}^{2}) = 0.
\end{equation}
To maximize $\beta$, we can rewrite and expand recursively in Eq. \ref{eq:proof1} :
\begin{equation}
    \I(\Tilde{P}^{1};Y) =\I(P^2, \Tilde{P}^{2};Y) =  \underbrace{\I(P^2;Y)}_{\text{maximizing}} + \underbrace{\I(\Tilde{P}^{2};Y)}_{\text{expanding}} - \underbrace{\I(P^2;\Tilde{P}^{2};Y)}_{0}.
\end{equation}
Therefore, it can be shown that:
\begin{equation}
    \I(\Tilde{P}^{k};Y) =\I(P^{k+1}, \Tilde{P}^{k+1};Y) =  \I(P^{k};Y) +\dots +\I(P^{K};Y) \;\;\; \forall k\in \{0,\dots,K-1\}.
\end{equation}
and for $k=0$, we have:
\begin{equation}
     \I(P^1, \dots, P^{K};Y) = \sum_{k=1}^{K}{\I(P^k;Y)},
\end{equation}
Finally, for maximizing $\I(P^1, \dots, P^{K};Y)$, we need to maximize each $\I(P^k;Y)$ so that each prototype feature contains Id-related information and is complemented. In other words, each prototype seeks to describe the input images from different aspects. 
\end{proof}

\subsection{Maximizing $\I(P^k;Y)$}
In Section \ref{sec:proposed}, the MI between the representation of each prototype and the label of persons are maximized by minimizing cross-entropy loss (see Eq: \ref{eq:pid} of the manuscript). This approximation is formulated as {\bf Proposition \ref{prop:cross}}. 
\begin{prop}
\label{prop:cross}
Let $P^k$ and $Y$ be random variables with domains $\mathcal{P}^k$ and $\mathcal{Y}$, respectively. Minimizing the conditional cross-entropy loss of predicted label $\hat{Y}$, denoted by $\mathcal{H}(Y; \hat{Y}|P^k)$, is equivalent to maximizing the $\I(P^k; Y)$
\end{prop}
\begin{proof}
    Let us define the MI as entropy,
    \begin{equation}
        \I(P^k;Y) = \underbrace{\mathcal{H}(Y)}_{\delta} - \underbrace{\mathcal{H}(Y|P^k)}_{\xi}
    \end{equation}
    Since the domain $\mathcal{Y}$ does not change, the entropy of the identity $\delta$ term is a constant and can therefore be ignored. Maximizing $\I(P^k,Y)$ can only be achieved by minimizing the $\xi$ term. We show that $\mathcal{H}(Y|P^k)$ is upper-bounded by our cross-entropy loss (Eq. \ref{eq:pid}), and minimizing such loss results in minimizing the $\xi$ term. By expanding its relation to the cross-entropy \cite{boudiaf2020unifying}:
    \begin{equation}
        \label{eq:cross}
        \mathcal{H}(Y; \hat{Y}|P^k) = \mathcal{H}(Y|P^k) + \underbrace{\mathcal{D}_{\text{KL}}(Y||\hat{Y}|P^k)}_{\geq 0} , 
    \end{equation}
    where:
    \begin{equation}
        \mathcal{H}(Y|P^k) \leq \mathcal{H}(Y; \hat{Y}|P^k).
    \end{equation}
    Through the minimization of Eq. \ref{eq:pid}, training can naturally be decoupled in 2 steps. First, weights of the prototype module are fixed, and only the classifier parameters (i.e., weight $W^k$ of the fully connected layer) are minimized w.r.t. Eq. \ref{eq:cross}. Through this step, $\mathcal{D}_{\text{KL}}(Y||\hat{Y}|P^k)$ is minimized by adjusting $\hat{Y}$ while the $\mathcal{H}(Y|P^k)$ does not change.
    In the second step, the prototype module’s weights are minimized w.r.t. $\mathcal{H}(Y|P^k)$, while the classifier parameters $W^k$ are fixed.
\end{proof}

\newpage
\section{Additional Details on the proposed method }
\subsection{Training Algorithm}

To train the model, BMDG uses a batch of data containing $N_b$ person with $N_p$ positive images from the V and I modalities. Algorithm \ref{alg:joint_learning} shows the details of the BMGD training strategy for optimizing the feature backbone by gradually increasing mixing prototypes.

At first, the prototype mining module extracts $K$ prototypes from infrared and visible images in lines 4 and 5 . Then, at lines 6 and 7, $\mathcal{G}$ function mixes these prototypes from each modality to create two intermediate features. It is noted that the ratio of mixing gradually increases w.r.t the step number $t$ to create more complex samples. To refine the final feature descriptor for input images, the attentive embedding module, $\mathcal{F}$, is applied to prototypes to leverage the attention between them.
At the end of each iteration, the model's parameters will be optimized by minimizing the cross-modality ReID objectives between each modality features vector and its gradually created intermediate.

\begin{algorithm}
\caption{BMDG Training Strategy.}
\label{alg:joint_learning}
\begin{algorithmic}[1]
\Require $\mathcal{S} = \{ \mathcal{V}, \mathcal{I}\}$ as training data
and $T, K$ as hyper-parameters
\For{$t = 1, \dots, T$} \Comment{over $T$ steps \ \ \ }
\While{all batches are not selected}
    \State $x_v^j, x_i^j \leftarrow \textbf{batchSampler}(N_b,N_p)$  
    \State extract prototypes $\mathbf{A}_v^{j}$ and global features $\mathbf{g}_v^{j}$ from visible images $v^j$ \Comment{left-side of Fig. \ref{fig:method}(a) \ \ } 
    \State extract prototypes $\mathbf{A}_i^{j}$ and global features $\mathbf{g}_i^{j}$ from infrared images $i^j$ \Comment{left-side of Fig. \ref{fig:method}(a) \ \ } 
    \State $\mathbf{A}_v^{(t)} \leftarrow \mathcal{G}({\mathbf{A}_v^j}, {\mathbf{A}_i^j},t)$ \Comment{V intermediate by gradually increasing the mixing ratio w.r.t $t$ from I prototypes  \ \ \ } 
    \State $\mathbf{A}_i^{(t)} \leftarrow \mathcal{G}({\mathbf{A}_i^j}, {\mathbf{A}_v^j},t)$ \Comment{I intermediate by gradually increasing the mixing ratio w.r.t $t$ from V prototypes} 
    \If{$t \leq T$}{:}
        \State $\mathbf{f}_v^{(t)} \leftarrow [\mathcal{F}(\mathbf{A}^{(t)}_v); \mathbf{g}_v]$ \Comment{embedding intermediate visible features \ \ \ } 
        \State $\mathbf{f}_i^{(t)} \leftarrow [\mathcal{F}(\mathbf{A}^{(t)}_i); \mathbf{g}_i]$ \Comment{embedding intermediate infrared features} 
    \Else{:}
        \State $\mathbf{f}_v^{(t)} \leftarrow [\mathcal{F}(\mathbf{A}^{j(t)}_v); \mathbf{g}_i]$ 
        \State $\mathbf{f}_i^{(t)} \leftarrow [\mathcal{F}(\mathbf{A}^{(t)}_i); \mathbf{g}_v]$ 
    \EndIf
    \State update model's parameters by optimizing Eq. \ref{eq:all_losses}
\EndWhile
\EndFor
\end{algorithmic}
\end{algorithm}

\subsection{Attentive Prototype Embedding}
Details of the Attentive Prototype Embedding module are depicted in \ref{fig:ape}. 
\begin{figure}[!b]
    \centering
    \includegraphics[width=0.50\textwidth]{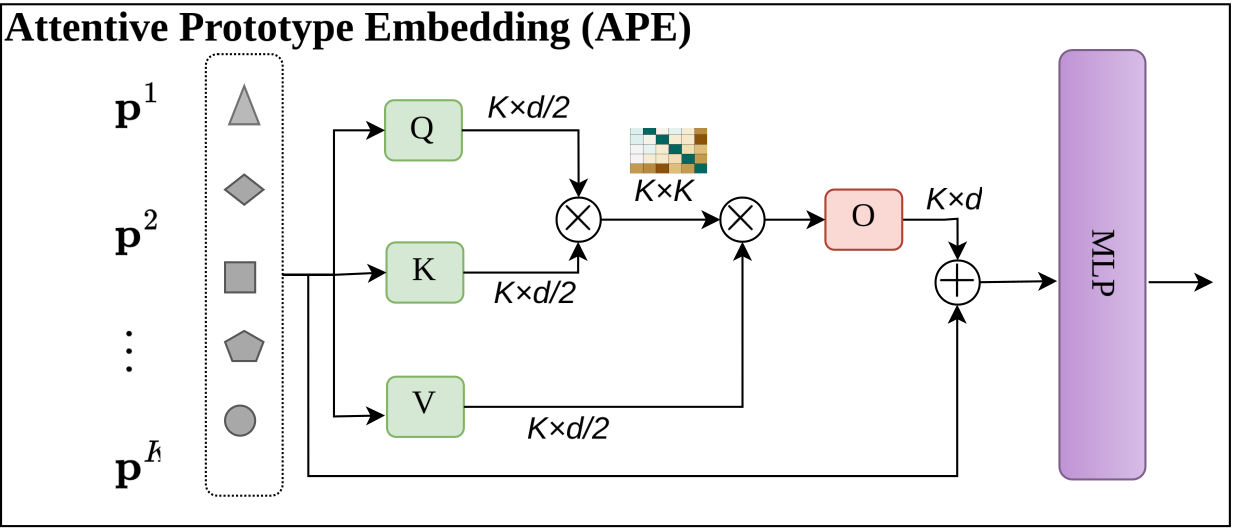}
    \caption{Attentive prototype embedding (APE) architecture.}
    \label{fig:ape}
\end{figure}

\section{Additional Details on the Experimental Methodology}


\subsection{Datasets:} 
Research on cross-modal V-I ReID has extensively used the SYSU-MM01 \cite{SYSU}, RegDB \cite{regDB}, and recently published LLCM \cite{LLCM} datasets. SYSU-MM01 is a large dataset containing more than 22K RGB and 11K IR images of 491 individuals captured with 4 RGB and 2 near-IR cameras, respectively. Of the 491 identities, 395 were dedicated to training, and 96 were dedicated to testing. Depending on the number of images in the gallery, the dataset has two evaluation modes: single-shot and multi-shot. RegDB contains 4,120 co-located V-I images of 412 individuals. Ten trial configurations randomly divide the dataset into two sets of 206 identities for training and testing. The tests are conducted in two ways -- comparing I to V (query) and vice versa.  An LLCM dataset consists of a large, low-light, cross-modality dataset that is divided into training and testing sets at a 2:1 ratio.

\subsection{Experimental protocol:}
We used a pre-trained ResNet50 \cite{resnet} as the deep backbone model. Each batch contains 8 RGB and 8 IR images from 10 randomly selected identities. Each image input is resized to 288 by 144, then cropped and erased randomly, and filled with zero padding or mean pixels. ADAM optimizer with a linear warm-up strategy was used for the optimization process. We trained the model by 180 epochs, in which the initial learning rate is set to 0.0004 and is decreased by factors of 0.1 and 0.01 at 80 and 120 epochs, respectively. $K=6$, $T=4$, $\lambda_\textit{f}=0.1$, $\lambda_\textit{v}=0.05$, $\lambda_\textit{p}=0.2$ and $\lambda_\textit{i}=0.4$ are set based on the analyses shown in the ablation study in the main paper and in Section \ref{sec:params}. $\lambda_\textit{eq}=0.5$ is for all experiments.  

\subsection{Performance measures:}
We use Cumulative Matching Characteristics (CMC) and Mean Average Precision (mAP) as assessment metrics in our study. In CMC, rank-k accuracy is measured to determine how likely it is that a precise cross-modality image of the person will be present in the top-k retrieved results. As an alternative, mAP can be used as a measure of image retrieval performance when multiple matching images are found in a gallery.

\section{Additional Quantitative Results}


\subsection{Hyperparameter values:} \label{sec:params}

This subsection analyzes the impact of $\lambda_\textit{f}$, $\lambda_\textit{v}$, $\lambda_\textit{p}$, and $\lambda_\textit{i}$ on V-I ReID accuracy. We initially set $\lambda_\textit{v}=0.01$, $\lambda_\textit{p}=0.05$, and $\lambda_\textit{i}=0.8$, experimenting with various values for $\lambda_\textit{f}$. As shown in Fig. \ref{fig:params}, accuracy improves with increasing $\lambda_\textit{f}$ until it reaches 0.1. Elevated $\lambda_\textit{f}$ enhances prototype diversity, boosting the discriminative ability of final features in the diverse space. However, excessively high values disperse prototype features in the feature space, diminishing discriminability and hindering accurate identification.

Similar trends are observed when $\lambda_\textit{p}=0.05$ and $\lambda_\textit{i}=0.8$, varying $\lambda_\textit{v}$ from 0.01 to 0.05, resulting in improved performance. Higher $\lambda_\textit{v}$ compresses prototype regions excessively, lacking sufficient ID-related information. Conversely, $\lambda_\textit{i}$ enhances the discriminative capabilities of prototypes in images. Balancing these factors, we find optimal values of 0.05 and 0.4 for $\lambda_\textit{v}$ and $\lambda_\textit{i}$, respectively. Additionally, based on experimentation, we set $\lambda_\textit{p}=0.2$ at the end of our analysis.
\begin{figure*} [!ht]
  \centering
  \begin{subfigure}{0.24\linewidth}
    \includegraphics[width=\linewidth]{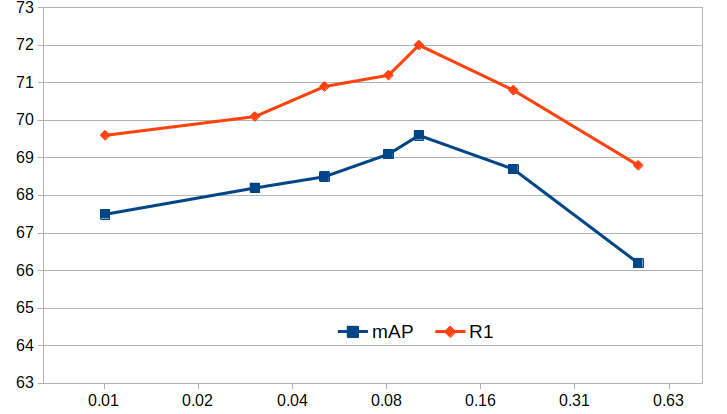}
    \caption{$\lambda_\textit{f}$.}
  \end{subfigure}
  \hfill
  \begin{subfigure}{0.24\linewidth}
    \includegraphics[width=\linewidth]{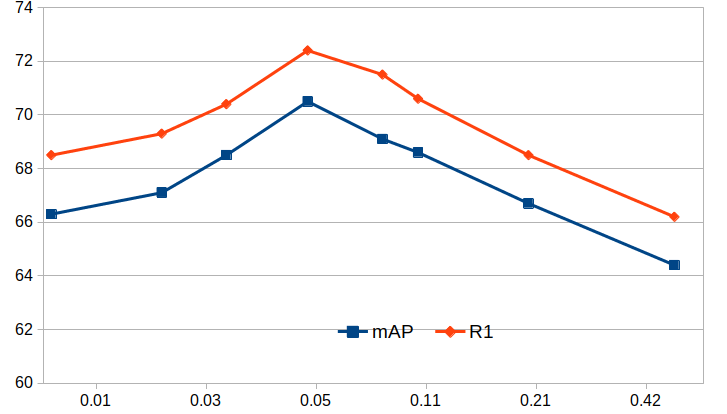}
    \caption{$\lambda_\textit{v}$.}
  \end{subfigure}
  \hfill
  \begin{subfigure}{0.24\linewidth}
    \includegraphics[width=\linewidth]{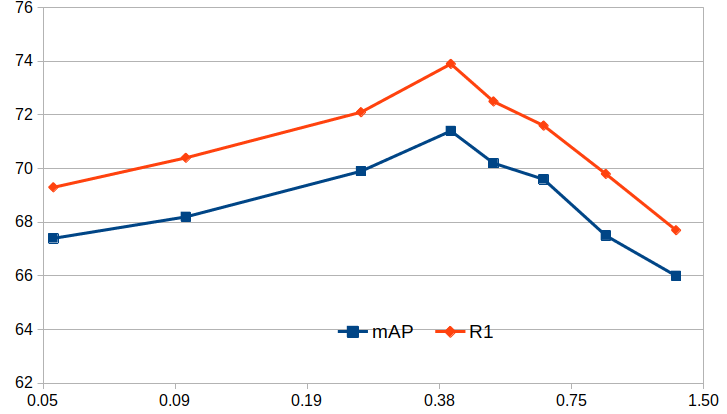}
    \caption{$\lambda_\textit{i}$.}
  \end{subfigure}
  \begin{subfigure}{0.24\linewidth}
    \includegraphics[width=\linewidth]{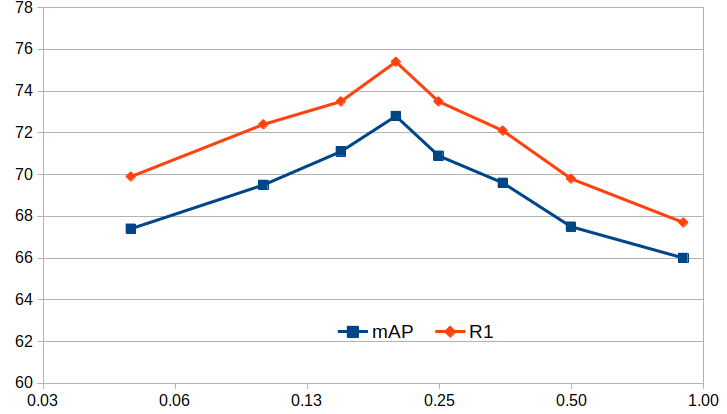}
    \caption{$\lambda_\textit{p}$.}
  \end{subfigure}
  \caption{Accuracy of the proposed BMDG over $\lambda_\textit{f}$, $\lambda_\textit{v}$, $\lambda_\textit{i}$, and $\lambda_\textit{p}$ values on SYSU-MM01 dataset in all-search and single-shot mode.}
  \label{fig:params}
\end{figure*}

%
\subsection{Step size and number of part prototypes:}
In section \ref{sec:Ablation}, we discussed the step size and number of part prototypes based on Rank-1 accuracy. Here we report the mAP measurement for Table \ref{tab:part_step} in paper in Table \ref{tab:part_step_mAP} and Table \ref{tab:part_step2} in paper in Table \ref{tab:part_step2_mAP}, respectively:

\begin{table}[!htb]
    \caption{mAP accuracy of BMDG using (a) our prototype mixing and (b) Mixup \cite{mixup} setting for different numbers of part prototypes ($K$) and intermediate steps ($T$).}
    \begin{subtable}{.5\linewidth}
      \centering
        \caption{Prototype exchanging}
        \label{tab:part_step_mAP}
            \resizebox{0.85\linewidth}{!}{
                \begin{tabular}{|c||cccccc|}
\hline
\multirow{2}{*}{$T$} & \multicolumn{6}{c|}{\textbf{Number of part prototypes ($K$)}} \\ \cline{2-7}
    &3     & 4     & 5     & 6              & 7     & 10    \\ \hline \hline
0  & 65.98 & 67.03 & 67.23 & 68.11          & 67.55 & 65.66 \\
1  & 67.42 & 68.72 & 69.28 & 69.46          & 68.32 & 69.28 \\
2  & 69.51 & 70.08 & 70.72 & 71.14          & 69.97 & 71.25 \\
3  & 71.44 & 71.69 & 71.82 & 72.02          & 71.06 & 71.67 \\
4  & -     & 71.98 & 72.15 & \textbf{72.86} & 71.19 & 69.00 \\
6  & -     & -     & -     & 72.40          & 71.17 & 69.54 \\
10 & -     & -     & -     & -              & -     & 69.46 \\ \hline
\end{tabular}
            }
    \end{subtable}%
    \begin{subtable}{.5\linewidth}
      \centering
        \caption{Mixup \cite{mixup}}
        \label{tab:part_step2_mAP}
         \resizebox{0.85\linewidth}{!}{
        \begin{tabular}{|c||cccccc|} \hline 
\multirow{2}{*}{$T$} & \multicolumn{6}{c|}{\textbf{Number of part prototypes ($K$)}} \\ \cline{2-7}
       &   3   & 4     & 5         & 6              & 7     & 10    \\ \hline  \hline
0      & 65.98 & 67.03 & 67.23     & 68.11          & 67.55 & 65.66 \\
1      & 66.31 & 67.22 & 67.31     & 68.28          & 68.5  & 65.8 \\
4      & 66.50 & 67.68 & 67.79     & 68.52          & 68.49 & 65.88 \\
6      & 66.98 & 67.77 & 68.05     & \textbf{68.73} & 68.56 & 66.02 \\
10     & 67.04 & 67.60 & 67.93     & 68.65          & 68.54 & 65.57 \\ \hline
\end{tabular}
        }
    \end{subtable} 
     \vspace{-0.5 cm}
\end{table}

\subsection{Model efficiency:} \label{sec:exp_gen}
Our BMDG proposed a method to extract alignable part-prototypes in feature extraction and then compute an attention embedding for the final representation features. In Table \ref{tab:infrence_complex}, we showed each component size and time complexity in the inference time compared to the baseline we used.

\begin{table}[ht]
\centering
\setlength\tabcolsep{2.5pt}
\caption{Number parameters and floating-point operations at inference time for BMDG and all its sub-modules. }
\label{tab:infrence_complex}
\begin{tabular}{|l|c|c|}
\hline
\textbf{Model} & \textbf{\# of Para. (M)} & \textbf{Flops (G)}  \\ \hline \hline

Feature Backbone  &  {24.8}  & {5.2}  \\ 
Prototype Mining  &  {3.1}  & {0.2}  \\ 
Attentive Prototype Embedding  &  {1.8}  & {0.3}  \\ 
\hline
baseline \cite{all-survey}  &  {24.9}   & {5.2}    \\
BMDG   &  {29.7}   & {5.7}    \\
\hline
\end{tabular}

\end{table}

\section{Visual Results}

\subsection{UMAP projections:}
To show the effectiveness of BMDG, we randomly select 7 identities from the SYSU-MM01 dataset and project their feature representations using the UMAP method \cite{mcinnes2018umap} for (a) Baseline, (b) one-step (prototypes without gradual training), and (c) BMDG. Visualization results (Figure \ref{fig:umap}) show that compared with the baseline and one-step approach, the feature representations learned with our BMDG method are well clustered according to their respective identity, showing a strong capacity to discriminate. BMDG is effective for learning robust and identity-aware features. Our BMDG method reduces this distance across modalities for each person and provides more separation among samples from different people. 

\begin{figure*} [!ht]
  \centering
  \begin{subfigure}{0.32\linewidth}
    \includegraphics[width=\linewidth]{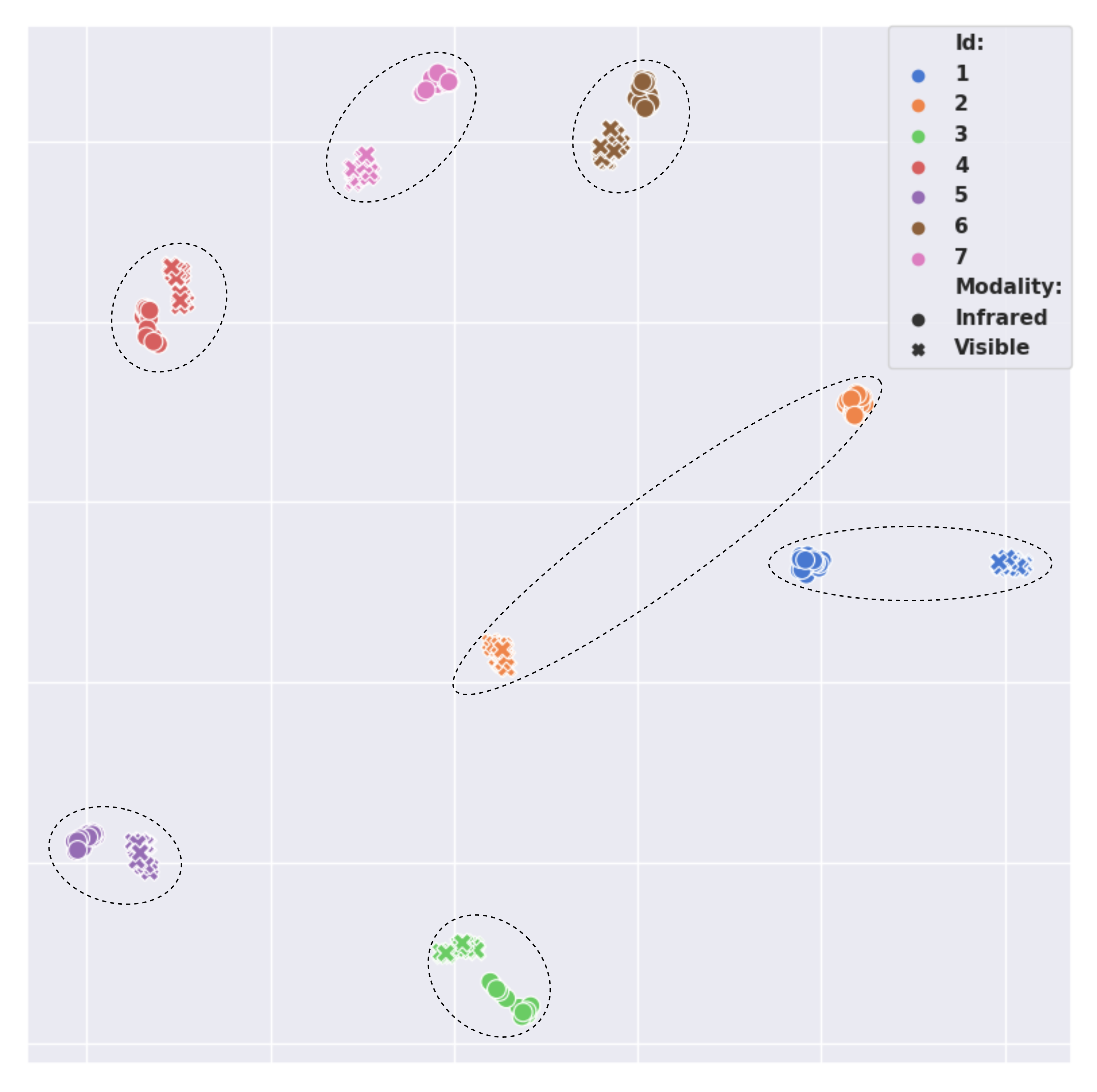}
    \caption{Baseline.}
  \end{subfigure}
  \hfill
  \begin{subfigure}{0.32\linewidth}
    \includegraphics[width=\linewidth]{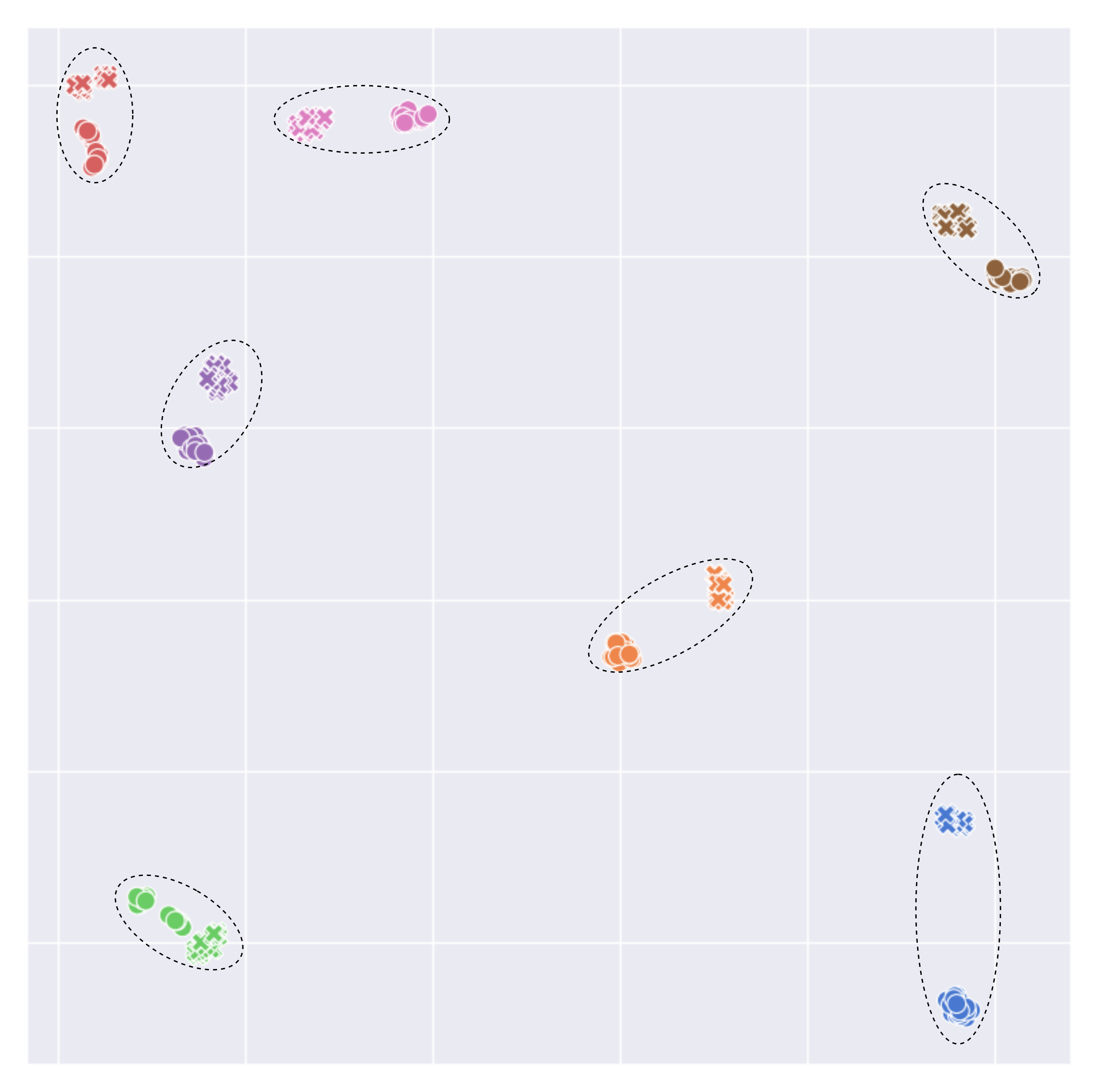}
    \caption{One Step.}
  \end{subfigure}
  \hfill
  \begin{subfigure}{0.32\linewidth}
    \includegraphics[width=\linewidth]{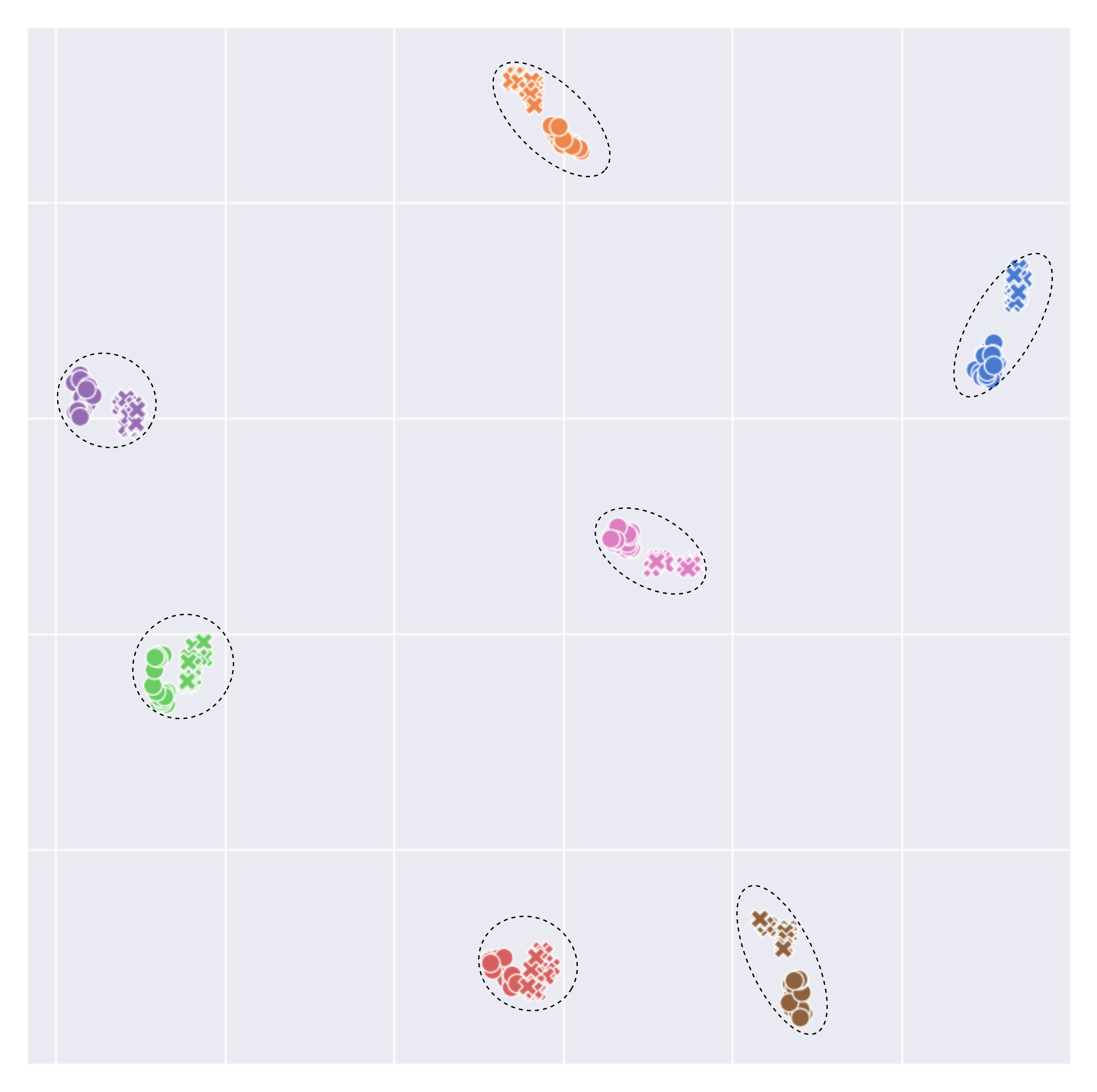}
    \caption{BMDG.}
  \end{subfigure}
  \caption{Distributions of learned V and infrared features of 7 identities from SYSU-MM01 dataset for (a) the baseline, (b) one-step using part-prototypes, and (c) our BMDG method by UMAP \cite{mcinnes2018umap}. Each color shows the identity. }
  \label{fig:umap}
\end{figure*}

Also, to show how the intermediate features gradually mix the modalities, we draw intermediate features for 6 steps in Figure \ref{fig:umap2}. At the beginning of training, the features are based on modality while at step 6, the features are concentrated on each identity. 

\begin{figure*} [!ht]
  \centering
  \begin{subfigure}{0.32\linewidth}
    \includegraphics[width=\linewidth]{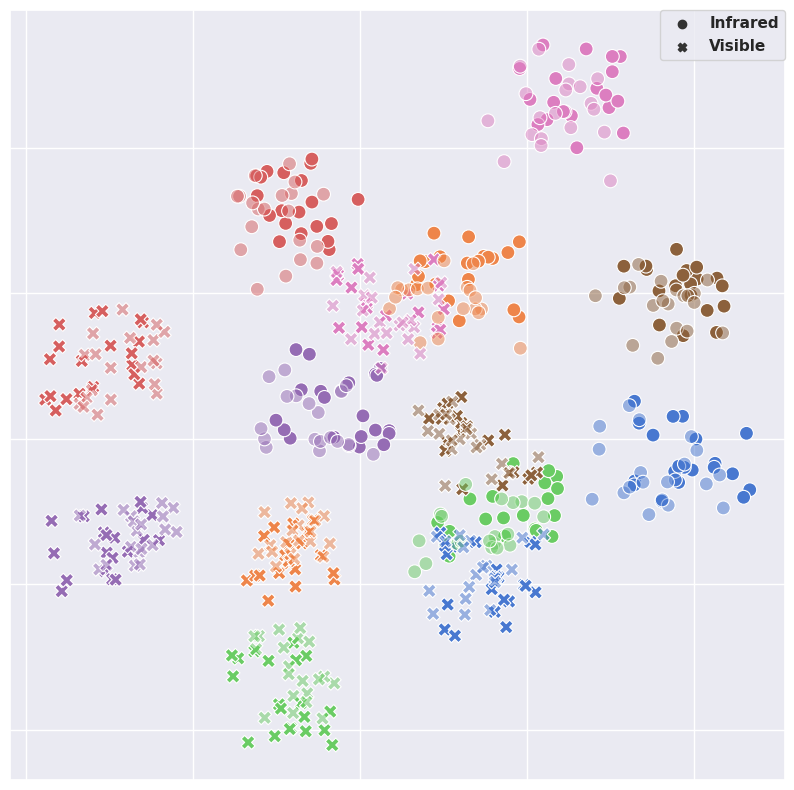}
    \caption{Step 1.}
  \end{subfigure}
  \hfill
  \begin{subfigure}{0.32\linewidth}
    \includegraphics[width=\linewidth]{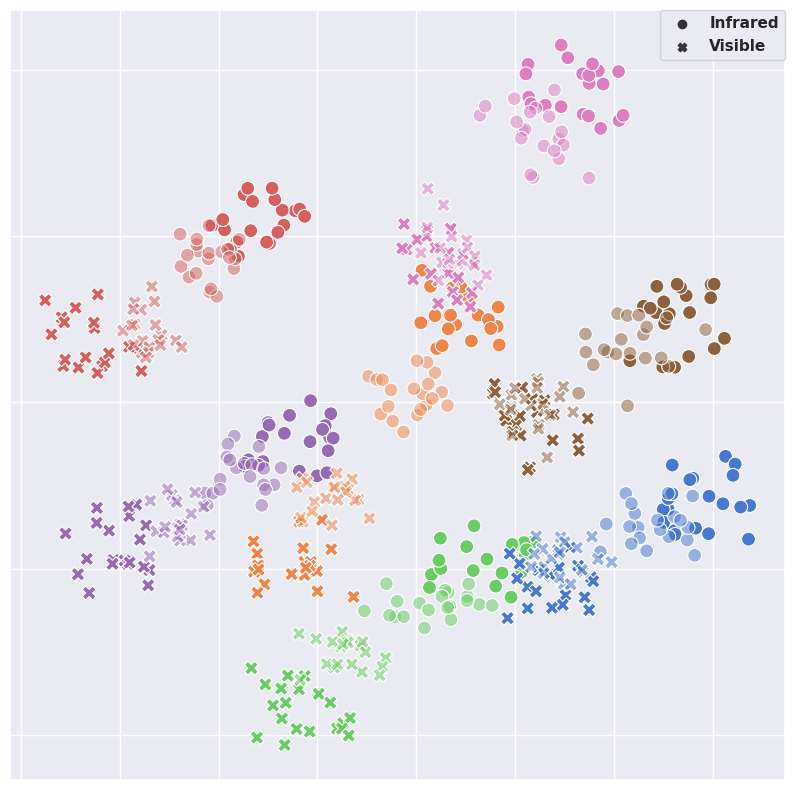}
    \caption{Step 2.}
  \end{subfigure}
  \hfill
  \begin{subfigure}{0.32\linewidth}
    \includegraphics[width=\linewidth]{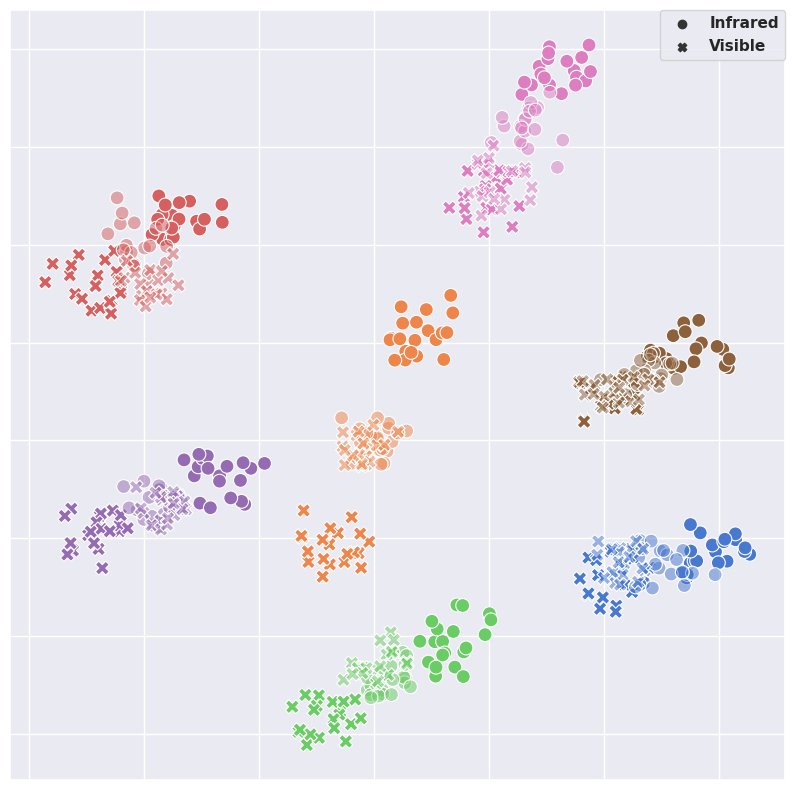}
    \caption{Step 3.}
  \end{subfigure}
  \\
  \begin{subfigure}{0.32\linewidth}
    \includegraphics[width=\linewidth]{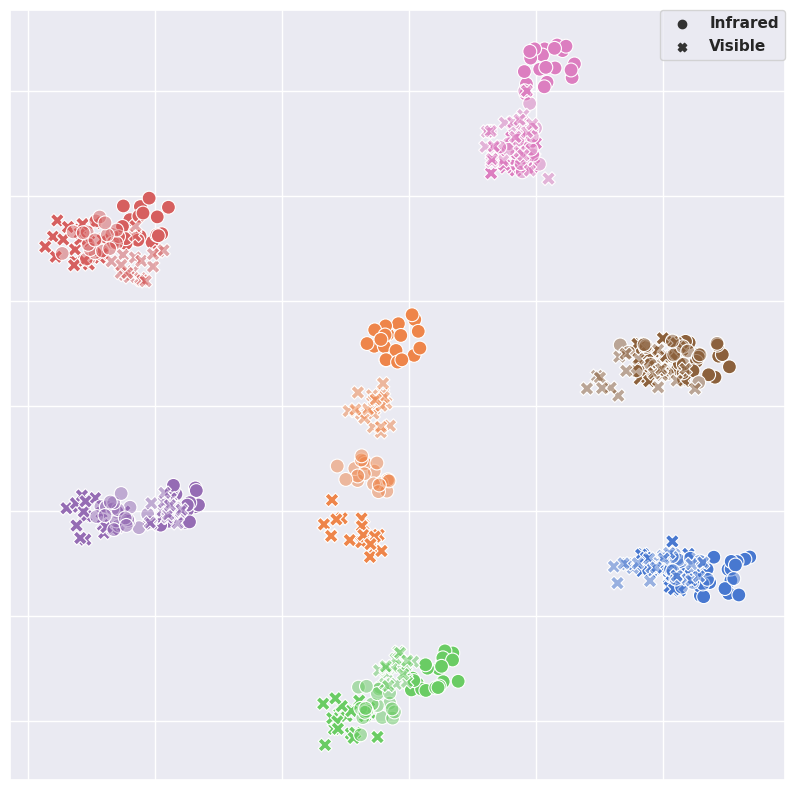}
    \caption{Step 4.}
  \end{subfigure}
  \hfill
  \begin{subfigure}{0.32\linewidth}
    \includegraphics[width=\linewidth]{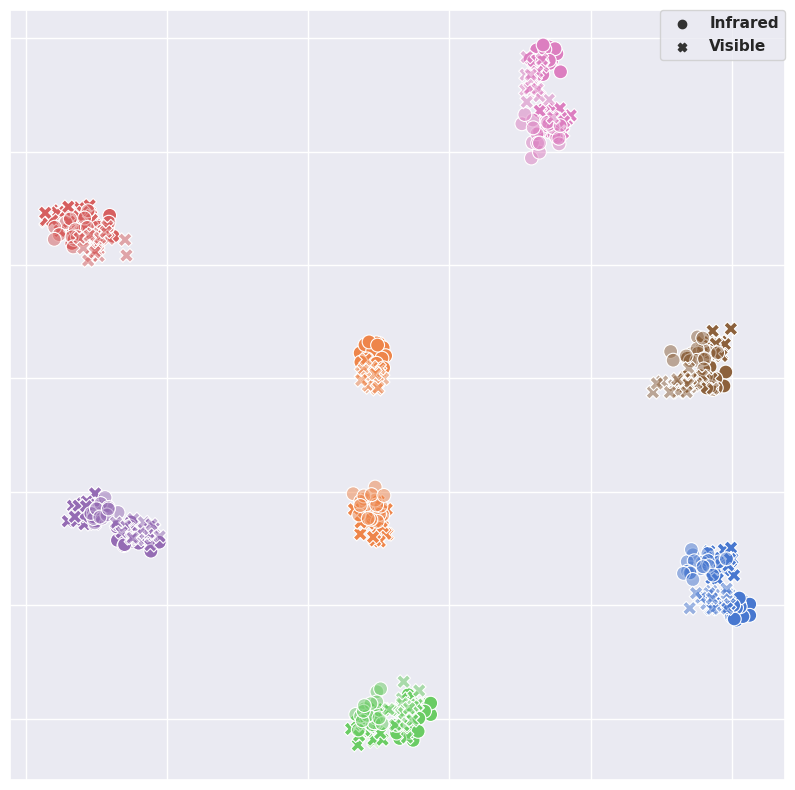}
    \caption{Step 5.}
  \end{subfigure}
  \hfill
  \begin{subfigure}{0.32\linewidth}
    \includegraphics[width=\linewidth]{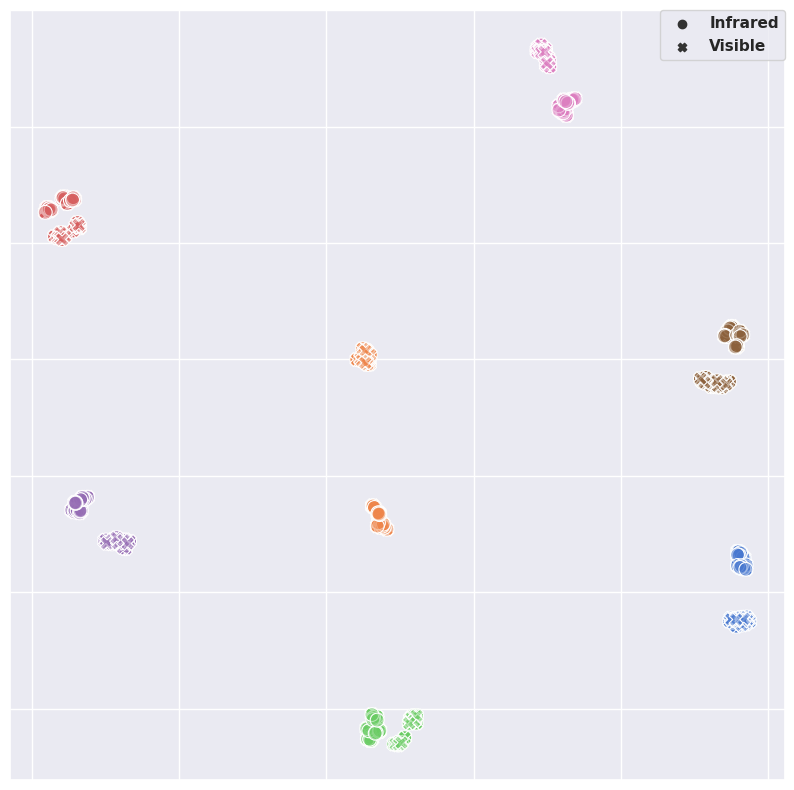}
    \caption{Step 6.}
  \end{subfigure}
  \caption{Distributions of learned V and infrared features of 7 identities from SYSU-MM01 dataset in training for 6 steps at epochs of 10,30,50,70,90 and 160 respectively by UMAP \cite{mcinnes2018umap}. Each color shows the identity. The intermediate features are drawn with lower opacity. }
  \label{fig:umap2}
\end{figure*}

\subsection{Domain shift:}

To estimate the level of domain shift over data from V and I modalities, we measured the MMD distance for each training epoch. To this end, for each epoch, we selected 10 random images from 50 random identities and extracted the prototype and global features, then measured the MMD distance between the centers of those features for each modality as shown in Fig. \ref{fig:mmd}(b). We report the normalized MMD distances between I and V features for our BMDG approach when compared with the baseline. Our method reduces this distance more than the y baseline. Thus, the results show that the intermediate domains improve the model robustness to a large multi-modal domain gap by gradually increasing the mix in prototypes over multiple steps.


\begin{figure*} [!ht]
  \centering
  \begin{subfigure}{0.45\linewidth}
    \includegraphics[width=\linewidth]{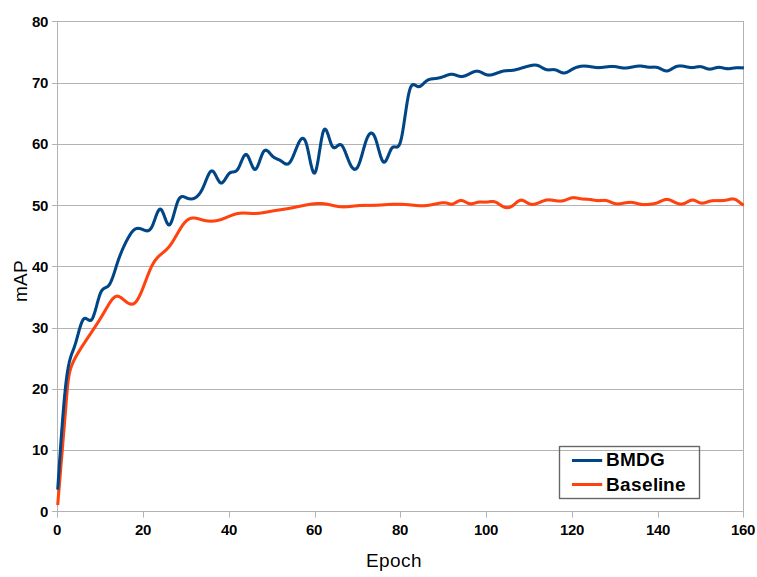}
    \caption{Learning Curve.}
  \end{subfigure}
  \hfill
  \begin{subfigure}{0.45\linewidth}
    \includegraphics[width=\linewidth]{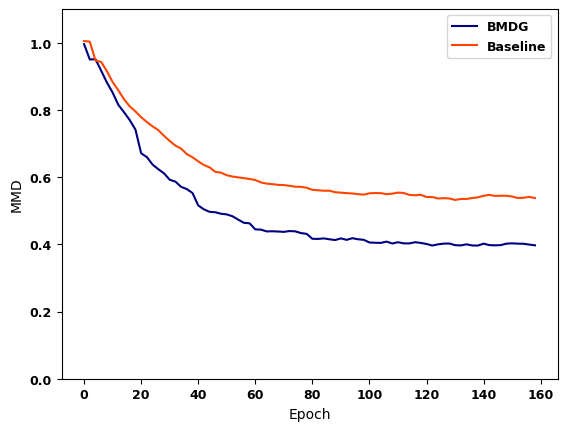}
    \caption{BMDG.}
  \end{subfigure}
  \caption{ The mAP and domain shift (MMD distance) between I and V modalities over training epochs. (a) The learning curve of BMDG vs Baseline\cite{all-survey}. (b) MMD distance over the center of multiple person's infrared features to visible modality. }
  \label{fig:mmd}
\end{figure*}

\subsection{Part-prototype masking:}
To show spatial information related to prototype features, we visualize the score map in the PM module (see Fig. \ref{fig:mask}). Our approach encodes prototype regions linked to similar body parts without considering person identity. Our model tries to find similar regions for each class of prototypes and then extracts ID-related information for that region. Therefore, BMDG is more robust for matching the same part features.
\begin{figure*} [!ht]
  \centering
  \begin{subfigure}{0.45\linewidth}
    \includegraphics[width=\linewidth,height=0.8\linewidth]{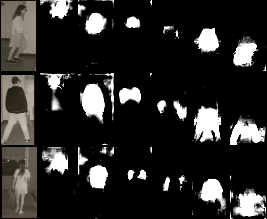}
    \caption{Infrared.}
  \end{subfigure}
  \begin{subfigure}{0.45\linewidth}
    \includegraphics[width=\linewidth]{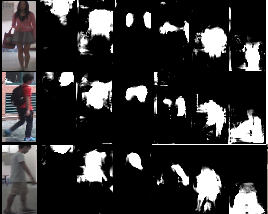}
    \caption{Visible.}
  \end{subfigure}
  \caption{ Prototypes regions extracted by the PRM module for (a) infrared and (b) visible images. Note that the mask size is 18$\times$9, which is then resized to fit the original input image. As shown, the mask of prototypes focuses on similar body parts without accounting for identity.}
  \label{fig:mask}
\end{figure*}

\subsection{Semi-supervised body part detection:}

An additional benefit of our HCL module lies in its ability to detect meaningful parts in a semi-supervised manner. By forcing the model to identify semantic regions that are both informative about foreground objects and contrastive to each other, our hierarchical contrastive learning provides robust part detection, even in the absence of part labels. To assess HCL, we fine-tuned our ReID model as a student using a pre-trained part detector \cite{part_detec} on the PASCAL-Part Dataset \cite{chen2014detect} as the teacher. In Fig. \ref{fig:unsuper_part}, the results show our model's strong capacity for detecting body parts compared to its teacher.
\begin{figure*} [!ht]
  \centering
  \begin{subfigure}{0.35\linewidth}
    \includegraphics[width=\linewidth, height=1.107\linewidth]{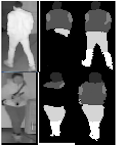}
    \caption{Infrared.}
  \end{subfigure}
  \begin{subfigure}{0.35\linewidth}
    \includegraphics[width=\linewidth,height=1.1\linewidth]{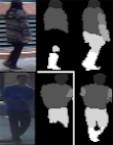}
    \caption{Visible.}
  \end{subfigure}
  \caption{Semi-supervised part discovery on the SYSU-MM01 dataset. The first columns are the (a) infrared and (b) visible images. The second column images are the result from \cite{part_detec}, and the last column are results with our fine-tuned model in BMDG. }
  \label{fig:unsuper_part}
\end{figure*}


\end{appendices}
\end{document}